%% file: averagefair.tex
\documentclass[11pt]{article}
\usepackage[utf8]{inputenc}
\usepackage{authblk}
\usepackage{amsgen,amsmath,amstext,amsbsy,amsopn,amssymb,bbm,listings,parskip,float, dsfont}
\usepackage{graphicx}
\usepackage{subcaption}
\usepackage[algoruled,boxed,lined]{algorithm2e}
\usepackage{tcolorbox}
\usepackage[font={footnotesize}]{caption}
\usepackage[english]{babel}
\usepackage{amsthm}
\usepackage{enumitem}
\usepackage{pifont}
\usepackage{multirow}
\usepackage{makecell}
\usepackage{mathabx}
\usepackage{microtype}

\usepackage{etoolbox}
\patchcmd{\thebibliography}{\section*{\refname}}{}{}{}

\newtheorem{theorem}{Theorem}[section]

\newtheorem{lemma}[theorem]{Lemma}

\newtheorem{remark}{Remark}[section]
\newtheorem{definition}{Definition}[section]
\newtheorem{assumption}{Assumption}[section]

\usepackage[breaklinks=true]{hyperref}
\usepackage{breakcites}
\hypersetup{
    colorlinks=true,
    linkcolor=blue,
    filecolor=magenta,
    urlcolor=cyan,
    citecolor = blue
}

\newenvironment{subroutine}[1][htb]
  {
   \begin{algorithm}[#1]
  }{\end{algorithm}}

\newenvironment{mapping}[1][htb]
  {
   \begin{algorithm}[#1]
  }{\end{algorithm}}

\DeclareMathOperator*{\argmax}{arg\,max}
\DeclareMathOperator*{\argmin}{arg\,min}

\newcommand{\R}{\mathbb{R}}
\newcommand{\N}{\mathbb{N}}

\newcommand{\Ps}{\mathbb{P}}
\newcommand{\E}{\mathbb{E}}

\newcommand{\1}{\mathds{1}}
\newcommand{\F}{\mathcal{F}}

\newcommand{\X}{\mathcal{X}}
\newcommand{\Px}{\mathcal{P}}
\newcommand{\Q}{\mathcal{Q}}
\newcommand{\Qhat}{\widehat{\mathcal{Q}}}
\newcommand{\Pxhat}{\widehat{\mathcal{P}}}
\newcommand{\Hs}{\mathcal{H}}

\newcommand{\Ls}{\mathcal{L}}

\newcommand{\lamb}{\boldsymbol{\lambda}}
\newcommand{\err}{\text{err}}

\newcommand{\Epsilon}{\mathcal{E}}

\newcommand{\gammahat}{\widehat{\gamma}}
\newcommand{\lambhat}{\widehat{\boldsymbol{\lambda}}}

\newcommand{\rhat}{\widehat{\boldsymbol{r}}}

\newcommand{\pvec}{\overset{}{\boldsymbol{p}}}
\newcommand{\pvechat}{\widehat{\boldsymbol{p}}}
\newcommand{\hvec}{\overset{}{\boldsymbol{h}}}

\newcommand{\wvec}{\overset{}{\boldsymbol{w}}}

\newcommand{\rvec}{\boldsymbol{r}}
\newcommand{\rvecfp}{\boldsymbol{r}_{\text{FP}}}

\newcommand{\What}{\widehat{W}}
\newcommand{\Epsilonfp}{\Epsilon_{\text{FP}}}
\newcommand{\Epsilonfn}{\Epsilon_{\text{FN}}}
\newcommand{\best}{\textbf{BEST}}
\newcommand{\bestfp}{\textbf{aBEST}_\text{FP}}
\newcommand{\Otilde}{\widetilde{O}}
\newcommand{\opt}{\text{OPT}}
\newcommand{\uniform}{\mathcal{U}}

\newcommand{\phat}{\widehat{p}}
\newcommand{\pr}{{p}^{(r)}}
\newcommand{\prhat}{{\widehat{p}}^{(r)}}
\newcommand{\pvechatr}{\widehat{\boldsymbol{p}}^{(r)}}

\newcommand{\pvecr}{{\boldsymbol{p}}^{(r)}}
\newcommand{\rhotilde}{\widetilde{\rho}}

\newcommand{\rhohat}{\widehat{\rho}}
\newcommand{\psihat}{\widehat{\psi}}
\newcommand{\rhohatvec}{\boldsymbol{\widehat{\rho}}}
\newcommand{\rhotildevec}{\boldsymbol{\widetilde{\rho}}}
\newcommand{\rhohatmin}{\widehat{\rho}_\text{min}}
\newcommand{\rhotildemin}{\rhotilde_{\text{min}}}
\newcommand{\rhomin}{\rho_{\text{min}}}

\title{Average Individual Fairness:\\ Algorithms, Generalization and Experiments}
\author{Michael Kearns}
\author{Aaron Roth}
\author{Saeed Sharifi-Malvajerdi}
\affil{University of Pennsylvania}
\date{\today}

\begin{document}
\maketitle

\begin{abstract}
We propose a new family of fairness definitions for classification problems that combine some of the best properties of both statistical and individual notions of fairness. We posit not only a distribution over \emph{individuals}, but also a distribution over (or collection of) \emph{classification tasks}. We then ask that standard statistics (such as error or false positive/negative rates) be (approximately) equalized across \emph{individuals}, where the rate is defined as an expectation over the classification tasks. Because we are no longer averaging over coarse groups (such as race or gender), this is a semantically meaningful individual-level constraint. Given a sample of individuals and classification problems, we design an oracle-efficient algorithm (i.e. one that is given access to any standard, fairness-free learning heuristic) for the fair empirical risk minimization task. We also show that given sufficiently many samples, the ERM solution generalizes in two directions: both to new individuals, and to new classification tasks, drawn from their corresponding distributions. Finally we implement our algorithm and empirically verify its effectiveness.
\end{abstract}

\newpage
{
\hypersetup{linkcolor=black}
\tableofcontents
}
\newpage

\section{Introduction}\label{sec:introduction}
\input{introduction-nips}

\subsection{Our Results}\label{sec:results}
\input{results-nips}

\subsection{Additional Related Work}\label{sec:related}
\input{related-nips}

\section{Model and Preliminaries}\label{sec:model}
\input{model}

\section{Learning subject to AIF}\label{sec:learningaif}
\input{AIF-algorithm}
\input{AIF-generalization}

\section{Experimental Evaluation}\label{sec:experiments}
\input{experiment-nips}

\section*{Acknowledgements}\label{sec:ack}
\addcontentsline{toc}{section}{\nameref{sec:ack}}
AR is supported in part by NSF grants AF-1763307 and CNS-1253345, and an Amazon Research Award.

\section*{References}\label{sec:ref}
\addcontentsline{toc}{section}{\nameref{sec:ref}}
\bibliographystyle{apalike}
\bibliography{averagefair}

\section*{Appendix}\label{sec:appendix}
\addcontentsline{toc}{section}{\nameref{sec:appendix}}
\appendix
\input{Appendix}


\end{document}

%% file: introduction-nips.tex
The community studying fairness in machine learning has yet to settle on definitions. At a high level, existing definitional proposals can be divided into two groups: \emph{statistical} fairness definitions and \emph{individual} fairness definitions. Statistical fairness definitions partition individuals into ``protected groups'' (often based on race, gender, or some other binary protected attribute) and ask that some statistic of a classifier (error rate, false positive rate, positive classification rate, etc.) be approximately equalized across those groups. In contrast, individual definitions of fairness have no notion of ``protected groups'', and instead ask for constraints that bind on pairs of individuals. These constraints can have the semantics that ``similar individuals should be treated similarly'' (\cite{awareness}), or that ``less qualified individuals should not be preferentially favored over more qualified individuals'' (\cite{JKMR16}). Both families of definitions have serious problems, which we will elaborate on. But in summary, statistical definitions of fairness provide only very weak promises to individuals, and so do not have very strong semantics. Existing proposals for individual fairness guarantees, on the other hand, have very strong semantics, but have major obstacles to deployment, requiring strong assumptions on either the data generating process or on society's ability to instantiate an agreed-upon fairness metric.

Statistical definitions of fairness are the most popular in the literature, in large part because they can be easily checked and enforced on arbitrary data distributions. For example, a popular definition (\cite{HPS16,KMR16,Chou17}) asks that a classifier's \emph{false positive rate} should be equalized across the protected groups. This can sound attractive: in settings in which a positive classification leads to a bad outcome (e.g. incarceration), it is the \emph{false positives} that are harmed by the errors of the classifier, and asking that the false positive rate be equalized across groups is asking that the harm caused by the algorithm should be proportionately spread across protected populations. But the meaning of this guarantee to an individual is limited, because the word \emph{rate} refers to an average over the population. To see why this limits the meaning of the guarantee, consider the example given in \cite{KNRW18}: imagine a society that is equally split between gender (Male, Female) and race (Blue, Green). Under the constraint that false positive rates be equalized across both race and gender, a classifier may incarcerate 100\% of blue men and green women, and 0\% of green men and blue women. This equalizes the false positive rate across all protected groups, but is cold comfort to any individual blue man and green woman. This effect isn't merely hypothetical --- \cite{KNRW18,KNRW19} showed similar effects when using off-the-shelf fairness constrained learning techniques on real datasets.

Individual definitions of fairness, on the other hand, can have strong individual level semantics. For example, the constraint imposed by \cite{JKMR16,JKMNR17} in online classification problems implies that the false positive rate must be equalized across all pairs of individuals who (truly) have negative labels. Here the word \emph{rate} has been redefined to refer to an expectation over the randomness of the classifier, and there is no notion of protected groups. This kind of constraint provides a strong individual level promise that one's risk of being harmed by the errors of the classifier are no higher than they are for anyone else. Unfortunately, in order to non-trivially satisfy a constraint like this, it is necessary to make strong realizability assumptions.

%% file: results-nips.tex
We propose an alternative definition of individual fairness that avoids the need to make assumptions on the data generating process, while giving the learning algorithm more flexibility to satisfy it in non-trivial ways. We consider that in many applications each individual will be subject to decisions made by \emph{many classification tasks} over a given period of time, not just one. For example, internet users are shown a large number of targeted ads over the course of their usage of a platform, not just one: the properties of the advertisers operating in the platform over a period of time are not known up front, but have some statistical regularities. Public school admissions in cities like New York are handled by a centralized match: students apply not just to one school, but to many, who can each make their own admissions decisions (\cite{abdulkadirouglu2005new}). We model this by imagining that not only is there an unknown distribution $\Px$ over individuals, but there is an unknown distribution $\Q$ over classification problems (each of which is represented by an unknown mapping from individual features to target labels). With this model in hand, we can now ask that the error rates (or false positive or negative rates) be equalized across all individuals --- where now \emph{rate} is defined as the average over \emph{classification tasks} drawn from $\Q$ of the probability of a particular individual being incorrectly classified.

We then derive a new oracle-efficient algorithm for satisfying this guarantee in-sample, and prove novel generalization guarantees showing that the guarantees of our algorithm hold also out of sample. Oracle efficiency is an attractive framework in which to circumvent the worst-case hardness of even \emph{unconstrained} learning problems, and focus on the \emph{additional computational difficulty} imposed by fairness constraints. It assumes the existence of ``oracles'' (in practice, implemented with a heuristic) that can solve weighted classification problems absent fairness constraints, and asks for efficient reductions from the fairness constrained learning problems to unconstrained problems. This has become a popular technique in the fair machine learning literature (see e.g. \cite{agarwal,KNRW18}) --- and one that often leads to practical algorithms. The generalization guarantees we prove require the development of new techniques because they refer to generalization in \emph{two orthogonal directions} --- over both individuals and classification problems. Our algorithm is run on a sample of $n$ individuals sampled from $\Px$ and $m$ problems sampled from $\Q$. It is given access to an oracle (in practice, implemented with a heuristic) for solving ordinary cost sensitive classification problems over some hypothesis space $\Hs$. The algorithm runs in polynomial time (it performs only elementary calculations except for calls to the learning oracle, and makes only a polynomial number of calls to the oracle) and returns a \emph{mapping from problems to hypotheses} that have the following properties, so long as $n$ and $m$ are sufficiently large (polynomial in the VC-dimension of $\Hs$ and the desired error parameters): For any $\alpha$, with high probability over the draw of the $n$ individuals from $\Px$ and the $m$ problems from $\Q$
\begin{enumerate}
\item \emph{Accuracy}: the error rate (computed in expectation over new individuals $x \sim \Px$ and new problems $f \sim \Q$) is within $O(\alpha)$ of the \emph{optimal} mapping from problems to classifiers in $\Hs$, subject to the constraint that for every pair of individuals $x,x'$ in the support of $\Px$, the error rates (or false positive or negative rates) (computed in expectation over problems $f \sim Q$) on $x$ and $x'$ differ by at most $\alpha$.
\item \emph{Fairness}: with probability $1-\beta$ over the draw of new individuals $x,x' \sim \Px$, the error rate (or false positive or negatives rates) of the output mapping (computed in expectation over problems $f \sim Q$) on $x$ will be within $O(\alpha)$ of that of $x'$.
\end{enumerate}
The mapping from new classification problems to hypotheses that we find is derived from the dual variables of the linear program representing our empirical risk minimization task, and we crucially rely on the structure of this mapping to prove our generalization guarantees for new problems $f \sim Q$.

%% file: related-nips.tex
The literature on fairness in machine learning has become much too large to comprehensively summarize, but see \cite{survey} for a recent survey. Here we focus on the most conceptually related work, which has aimed to bridge the gap between the immediate applicability of statistical definitions of fairness with the strong individual level semantics of individual notions of fairness. One strand of this literature focuses on the ``metric fairness'' definition first proposed by \cite{awareness}, and aims to ease the assumption that the learning algorithm has access to a task specific fairness metric. \cite{KRR18} imagine access to an oracle which can provide unbiased estimates to the metric distance between any pair of individuals, and show how to use this to satisfy a statistical notion of fairness representing ``average metric fairness'' over pre-defined groups. \cite{GJKR18} study a contextual bandit learning setting in which a human judge points out metric fairness violations whenever they occur, and show that with this kind of feedback (under assumptions about consistency with a family of metrics), it is possible to quickly converge to the optimal fair policy.  \cite{YR18} consider a PAC-based relaxation of metric fair learning, and show that empirical metric-fairness generalizes to out-of-sample metric fairness. Another strand of this literature has focused on mitigating the problems that arise when statistical notions of fairness are imposed over coarsely defined groups, by instead asking for statistical notions of fairness over exponentially many or infinitely many groups with a well defined structure. This line includes \cite{multical} (focusing on calibration), \cite{KNRW18} (focusing on false positive and negative rates), and \cite{KGZ18} (focusing on error rates).

%% file: model.tex

We model each individual in our framework by a vector of features $x \in \X$, and we let each learning problem\footnote{we will use the terms: problem, task, and labeling interchangeably.} be represented by a binary function $f \in \F$ mapping $\X$ to $\{0,1\}$. We assume probability measures $\Px$ and $\Q$ over the space of individuals $\X$ and the space of problems $\F$, respectively. Without loss of generality, we assume throughout that the support of the distributions $\Px$ and $\Q$ are $\X$ and $\F$, respectively. In the training phase there is a \emph{fixed} (across problems) set $X = \{x_i\}_{i=1}^n$ of $n$ individuals sampled independently from $\Px$ for which we have available labels of $m$ learning tasks represented by $F = \{f_j\}_{j=1}^m$ drawn independently from $\Q$\footnote{Throughout we will use subscript $i$ to denote individuals and $j$ to denote learning problems.}. Therefore, a training data set of $n$ individuals $X$  for $m$ learning tasks $F$ takes the form:
$
S = \big\{ x_i, \left( f_j(x_i) \right)_{j=1}^m \big\}_{i=1}^n
$.
We summarize the notations we use for individuals and problems in Table \ref{tab:sumnotations}.

\begin{table}
\caption{Summary of notations for individuals vs. problems.}
\renewcommand{\arraystretch}{1.3}
\centering
\begin{tabular}{|c|c|c|c|c|c|c|}
\hline
 & space & element & distribution & data set & sample size & empirical dist. \\
\hline
individual & $\X$ &  $x \in \X$ & $\Px$ & $X = \{x_i \}_{i=1}^n$ & $n$ & $\Pxhat = \uniform \left( X \right)$\\
\hline
problem & $\F$ & $f \in \F$ & $\Q$ & $F = \{f_j \}_{j=1}^m$ & $m$ & $\Qhat = \uniform \left( F \right)$\\
\hline
\end{tabular}
\label{tab:sumnotations}
\end{table}

In general the function class $\F$ will be unknown. We will aim to solve the (agnostic) learning problem over a hypothesis class $\Hs$, which need bear no relationship to $\F$. We assume throughout that $\Hs$ contains the constant classifiers $h^0$ and $h^1$ where $h^0 (x) = 0$ and $h^1 (x) = 1$ for all $x$. This assumption will allow us to argue about feasibility of the constrained optimization problems that we will solve. We allow for randomized classifiers, which we model as learning over $\Delta(\Hs)$, the probability simplex over $\Hs$.

Unlike usual learning settings where the primary goal is to learn a single hypothesis $p \in \Delta(\Hs)$, our objective is to learn a \emph{mapping} $\psi \in \Delta (\Hs)^\F$ that maps learning tasks $f \in \F$ represented as labelings of the training data to hypotheses $p \in \Delta (\Hs)$. We will therefore have to formally define the error rates incurred by a mapping $\psi$ and use them to formalize a learning task subject to our proposed fairness notion. For a mapping $\psi$, we write $\psi_f$ to denote the classifier corresponding to $f \in \F$ under the mapping, i.e., $\psi_f = \psi \left(f \right) \in \Delta(\Hs)$.

Notice in the training phase, there are only $m$ learning problems $F = \{f_j\}_{j=1}^m$ to be solved, and therefore, the corresponding empirical fair learning problem reduces to learning $m$ randomized classifiers $\pvec = \left( p_1, p_2, \ldots , p_m\right) \in \Delta (\Hs)^m$, where $p_j$ is the learned classifier for the $j$th problem $f_j \in F$. In general, learning $m$ specific classifiers for the training problems might not give any generalizable rule mapping new problems to hypotheses --- but the specific algorithm we propose will, in the form of a set of weights (derived from the dual variables of the ERM problem) over the training individuals.
\medskip

\subsection{AIF: Average Individual Fairness}
\begin{definition}[Individual and Overall Error Rates of a Mapping $\psi$]\label{def:individualerror}
For a given individual $x \in \X$, a mapping $\psi \in \Delta (\Hs)^\F$, and distributions $\Px$ and $\Q$ over $\X$ and $\F$, the individual error rate of $x$ incurred by $\psi$ is defined as follows:
$$
\Epsilon \left(x, \psi; \Q \right) = \underset{f \sim \Q}{\E} \left[ \underset{h \sim \psi_f}{\Ps} \left[ h(x) \neq f(x) \right] \right]
$$
The overall error rate of $\psi$ is:
$$
\err \left( \psi ; \Px, \Q\right) = \underset{x \sim \Px}{\E} \left[ \Epsilon \left(x, \psi; \Q \right) \right]
$$
\end{definition}
\medskip

In the body of this paper, we will focus on a fairness constraint that asks that the individual error rate should be approximately equalized across all individuals. In Section \ref{sec:learningfpaif} of the Appendix, we extend our techniques to equalizing false positive and negative rates across individuals. 
\medskip

\begin{definition}[Average Individual Fairness (AIF)]\label{def:fairness}
In our framework, we say a mapping $\psi \in \Delta (\Hs)^\F$ satisfies ``$(\alpha, \beta)$-AIF" (reads $(\alpha,\beta)$-approximate Average Individual Fairness) with respect to the distributions $\left( \Px, \Q \right)$ if there exists $\gamma \ge 0$ such that:
\begin{equation*}
\underset{x \sim \Px}{\Ps} \left( \left\vert \Epsilon \left(x,\psi ; \Q \right) - \gamma \right\vert > \alpha \right) \le \beta
\end{equation*}
\end{definition}
\medskip

\subsection{Notations}
We briefly fix some notations:
\begin{itemize}
\item For an event $A$, $\1 \left[ A \right]$ represents the indicator function of $A$. $\1 \left[ A \right] = 1$ if $A$ occurs.
\item For a natural number $n \in \N$, $[n] = \{ 1, 2 , \ldots, n \}$.
\item $\uniform \left(S \right)$ represents the uniform distribution over the set $S$.
\item For a mapping $\psi: A \to B$ and $A' \subseteq A$, $\psi \vert_{A'}$ represents $\psi$ restricted to the domain $A'$.
\item For a hypothesis class $\Hs$, $\Delta (\Hs)$ represents the probability simplex over $\Hs$.
\item $d_\Hs$ denotes the VC dimension of the hypothesis class $\Hs$.
\item $CSC (\Hs)$ denotes a \emph{cost sensitive classification oracle} for $\Hs$ which is defined below.
\end{itemize}

\medskip
\begin{definition}[Cost Sensitive Classification (CSC) in $\Hs$]\label{def:csc}
Let $D = \{x_i, c_i^1, c_i^0 \}_{i=1}^n$ denote a data set of $n$ individuals $x_i$ where $c_i^1$ and $c_i^0$ are the costs of classifying $x_i$ as positive (1) and negative (0) respectively. Given $D$, the cost sensitive classification problem defined over $\Hs$ is the optimization problem:
$$
\argmin_{h \in \Hs} \sum_{i=1}^n \left\{ c_i^1 h(x_i) + c_i^0 \left( 1 - h(x_i) \right) \right\}
$$
An oracle $CSC(\Hs)$ takes the data set $D = \{x_i, c_i^1, c_i^0 \}_{i=1}^n$ as input and outputs the solution to the optimization problem. We use $CSC (\Hs; D)$ to denote the classifier returned by $CSC (\Hs)$ on an input data set $D$. We say that an algorithm is oracle efficient if it runs in polynomial time given the ability to make unit-time calls to $CSC(\Hs)$. 
\end{definition}


%% file: AIF-algorithm.tex
In this section we first cast the learning problem subject to the AIF fairness constraints as the constrained optimization problem (\ref{box:fair}) and then develop an oracle efficient algorithm for solving its corresponding empirical risk minimization (ERM) problem (in the spirit of \cite{agarwal}). In the coming sections we give a full analysis of the developed algorithm including its \emph{in-sample} accuracy/fairness guarantees, and define the mapping it induces from new problems to hypotheses, and finally establish \emph{out-of-sample} bounds for this trained mapping.


\begin{tcolorbox}[title= {Fair Learning Problem subject to ($\alpha , 0$)-AIF}]
\begin{equation}\label{box:fair}
\begin{aligned}
& \min_{\psi \, \in \, \Delta(\Hs)^\F, \, \gamma \, \in \, [0,1]}  & & \err \left( \psi ; \Px, \Q \right) \\
& \ \text{ s.t. $\forall x \in \X$:} & & \left\vert \Epsilon \left(x, \psi; \Q \right) - \gamma \right\vert \le \alpha
\end{aligned}
\end{equation}
\end{tcolorbox}

\medskip
\begin{definition}[\opt]\label{def:opt}
Consider the optimization problem (\ref{box:fair}). Given distributions $\Px$ and $\Q$, and fairness approximation parameter $\alpha$, we denote the optimal solutions of (\ref{box:fair}) by $\psi^\star \left( \alpha ; \Px , \Q \right) $ and $\gamma^\star \left( \alpha ; \Px , \Q \right)$, and the value of the objective function at these optimal points by $\opt \left(\alpha; \Px, \Q \right)$. 
In other words
\begin{equation*}
\opt \left (\alpha; \Px, \Q \right) = \err \left( \psi^\star ; \Px, \Q \right)
\end{equation*}
\end{definition}

We will use $\opt$ as the benchmark with respect to which we evaluate the accuracy of our trained mapping. It is worth noticing that the optimization problem (\ref{box:fair}) has a nonempty set of feasible points for any $\alpha$ and any distributions $\Px$ and $\Q$ because the following point is always feasible: $\gamma = 0.5$ and $\psi_f = 0.5 h^0 + 0.5 h^1$ (i.e. random classification) for all $f \in \F$ where $h^0$ and $h^1$ are all-zero and all-one constant classifiers.

\subsection{The Empirical Fair Learning Problem}\label{subsec:empirical}
We start to develop our algorithm by defining the \emph{empirical} version of (\ref{box:fair}) for a given training data set of $n$ individuals $X = \{x_i\}_{i=1}^n$ and $m$ learning problems $F = \{f_j\}_{j=1}^m$. We will formulate the empirical problem as finding a \emph{restricted} mapping $\psi \vert_F$ by which we mean the domain of the mapping is restricted to the training set $F \subseteq \F$. We will later see how the dynamics of our proposed algorithm allows us to extend the restricted mapping to a mapping from the entire space $\F$. We slightly change notation and represent a restricted mapping $\psi \vert_F \in \Delta(\Hs)^F$ explicitly by a vector $\pvec = \left( p_1, \, \ldots, \, p_m \right) \in \Delta (\Hs)^m$ of randomized classifiers where $p_j \in \Delta (\Hs)$ corresponds to $f_j \in F$. Notice the empirical versions of the overall error rate and the individual error rates incurred by the mapping $\pvec$ (see Definition \ref{def:individualerror}) can be expressed as:
\begin{equation}\label{eq:empiricalerr}
\err \left( \pvec ; \Pxhat, \Qhat\right) = \frac{1}{n} \sum_{i=1}^n \Epsilon \left(x_i, \pvec ; \Qhat\right) = \frac{1}{m} \sum_{j=1}^m \frac{1}{n} \sum_{i=1}^n \underset{h_j \sim \, p_j}{\Ps} \left[ h_j (x_i) \neq f_j(x_i) \right]
\end{equation}
\begin{equation}\label{eq:empiricaleps}
\Epsilon \left(x, \pvec ; \Qhat  \right) = \frac{1}{m} \sum_{j=1}^m \underset{h_j \sim \, p_j}{\Ps} \left[ h_j (x) \neq f_j(x) \right]
\end{equation}

Using these empirical quantities, we cast the empirical version of the fair learning problem (\ref{box:fair}) as the constrained optimization problem (\ref{box:fairerm}) where there is one constraint for each individual in the training data set. Note that the optimization problem (\ref{box:fairerm}) forms a linear program and that we considered a slightly relaxed version of (\ref{box:fair}) where instead of $(\alpha, 0)$-AIF, we require $(2\alpha, 0)$-AIF (of course now with respect to the empirical distributions) only to make sure the optimal solution $\left( \psi^\star, \gamma^\star \right)$ of (\ref{box:fair}) (in fact $\psi^\star$ restricted to $F$) is feasible in (\ref{box:fairerm}) as long as enough samples are acquired. This will appear later in the generalization analysis of our proposed algorithm.

\begin{tcolorbox}[title= Empirical Fair Learning Problem]
\begin{equation}\label{box:fairerm}
\begin{aligned}
& \ \ \min_{\pvec \, \in \, \Delta(\Hs)^m, \, \gamma \, \in \, [0,1]}  & & \err \left( \pvec ; \Pxhat, \Qhat\right) \\
& \text{ s.t. $\forall i \in \{1, \ldots, n\}$:} & &\left\vert \Epsilon \left( x_i , \pvec ; \Qhat\right) - \gamma \right\vert \le 2 \alpha \\
\end{aligned}
\end{equation}
\end{tcolorbox}


\subsection{A Reductions Approach: Formulation as a Two-player Game}\label{subsec:noregret}
We use the dual perspective of constrained optimization problems to reduce the fair ERM (\ref{box:fairerm}) to a two-player game between a ``Learner" (primal player) and an ``Auditor" (dual player). Towards deriving the Lagrangian of (\ref{box:fairerm}), we first rewrite its constraints in $\rvec ( \pvec , \gamma; \Qhat ) \le 0$ form where
\begin{equation}\label{eq:r}
\rvec \left( \pvec , \gamma ; \Qhat \right) = \begin{bmatrix} \Epsilon \left(x_i, \pvec; \Qhat \right) - \gamma - 2\alpha \\ \gamma - \Epsilon \left(x_i, \pvec; \Qhat\right) - 2 \alpha \end{bmatrix}_{i=1}^n \in \R^{2n}
\end{equation}
represents the ``fairness violations" of the pair $\left( \pvec, \gamma \right)$ in one single vector. Let the corresponding dual variables be represented by $\lamb = \left[\lamb_i^+, \lamb_i^- \right]_{i=1}^n \in \Lambda$, where $\Lambda = \{ \lamb \in  \R^{2n}_+ \, \vert \, || \lamb ||_1 \le B \}$. Note we place an upper bound $B$ on the $\ell_1$-norm of $\lamb$ in order to reason about the convergence of our proposed algorithm. $B$ will eventually factor into both the run-time and the approximation guarantees of our solution. Using Equation (\ref{eq:r}) and the introduced dual variables, we have that the Lagrangian of (\ref{box:fairerm}) is
\begin{align}\label{eq:lagrangian}
\Ls \left(\pvec, \gamma, \lamb \right) &=\err \left( \pvec ; \Pxhat, \Qhat\right) + \lamb^T \rvec \left( \pvec , \gamma ; \Qhat \right)
\end{align}
We therefore consider solving the following minmax problem:
\begin{equation}\label{eq:minmax}
\min_{\pvec \, \in \, \Delta(\Hs)^m, \, \gamma \, \in \, [0,1]} \, \max_{\lambda \in \Lambda} \ \Ls \left(\pvec, \gamma, \lamb \right)  \ = \ \max_{\lambda \in \Lambda} \, \min_{\pvec \, \in \, \Delta(\Hs)^m, \, \gamma \, \in \, [0,1]} \  \Ls \left(\pvec, \gamma, \lamb \right)
\end{equation}
where strong duality holds because $\Ls$ is linear in its arguments and the domains of $(\pvec, \gamma)$ and $\lamb$ are convex and compact (\cite{sion}). From a game theoretic perspective, the solution to this minmax problem can be seen as an equilibrium of a zero-sum game between two players. The primal player (Learner) has strategy space $\Delta(\Hs)^m \times [0,1]$ while the dual player (Auditor) has strategy space $\Lambda$, and given a pair of chosen strategies $\left(\pvec, \gamma, \lamb \right)$, the Lagrangian $\Ls \left(\pvec, \gamma, \lamb \right)$ represents how much the Learner has to pay to the Auditor --- i.e. it defines the payoff function of a zero-sum game in which the Learner is the minimization player, and the Auditor is the maximization player.

Using no regret dynamics, an approximate equilibrium of this zero-sum game can be found in an iterative framework. In each iteration, we let the dual player run the \emph{exponentiated gradient descent} algorithm and the primal player \emph{best respond}. The best response problem of the Learner can be decoupled into $(m+1)$ separate minimization problems and that in particular, the optimal classifiers $\pvec$ can be viewed as the solutions to $m$ \emph{weighted classification problems} in $\Hs$ where all $m$ problems share the same weights $\wvec = [\lambda_i^+ - \lambda_i^-]_i \in \R^n$ over the training individuals. In the following subsection, we derive and implement the best response of the Learner where we use the learning oracle $CSC(\Hs)$ (see Definition \ref{def:csc}) to solve the weighted classification problems.

\subsection{\best: The Learner's Best Response}\label{subsec:bestresponse}
We formally describe and analyze the best response problem of the Learner in this subsection and summarize the results in a subroutine called $\best$.  In each iteration of the described iterative framework, the Learner is given some $\lamb \in \Lambda$ picked by the Auditor and she wants to solve the following minimization problem.
$$
\argmin_{\pvec \, \in \, \Delta(\Hs)^m, \, \gamma \, \in \, [0,1]} \, \Ls \left( \pvec, \gamma, \lamb \right)
$$
We will use Equations (\ref{eq:empiricalerr}) and (\ref{eq:empiricaleps}) to expand the Lagrangian (\ref{eq:lagrangian}) and decouple the above minimization problem into $(m+1)$ minimization problems, each depends only on one variable the Learner has to pick.

\begin{align*}
&\argmin_{\pvec \, \in \, \Delta(\Hs)^m, \, \gamma \, \in \, [0,1]} \, \Ls \left( \pvec, \gamma, \lamb \right) \\
\equiv &\argmin_{\pvec \, \in \, \Delta(\Hs)^m, \, \gamma \, \in \, [0,1]} \, \err \left( \pvec ; \Pxhat, \Qhat\right) + \sum_{i=1}^n \left\{ \lambda_{i}^+ \left( \Epsilon \left(x_i, \pvec; \Qhat \right) - \gamma  \right) + \lambda_{i}^- \left( \gamma  - \Epsilon \left(x_i, \pvec; \Qhat \right)  \right) \right\} \\
\equiv &\argmin_{\pvec \, \in \, \Delta(\Hs)^m, \, \gamma \, \in \, [0,1]} \, \frac{1}{n} \sum_{i=1}^n \Epsilon \left(x_i, \pvec ; \Qhat\right) + \sum_{i=1}^n \left\{ \lambda_{i}^+ \left( \Epsilon \left(x_i, \pvec; \Qhat \right) - \gamma  \right) + \lambda_{i}^- \left( \gamma  - \Epsilon \left(x_i; \pvec; \Qhat \right)  \right) \right\} \\
\equiv &\argmin_{\pvec \, \in \, \Delta(\Hs)^m, \, \gamma \, \in \, [0,1]} \, \sum_{i=1}^n \left\{ \lambda_{i}^- - \lambda_{i}^+ \right\}  \gamma + \sum_{i=1}^n \left(\frac{1}{n} + \lambda_{i}^+ - \lambda_{i}^- \right) \Epsilon \left(x_i, \pvec; \Qhat \right) \\
\equiv &\argmin_{\pvec \, \in \, \Delta(\Hs)^m, \, \gamma \, \in \, [0,1]} \, \sum_{i=1}^n \left\{ \lambda_{i}^- - \lambda_{i}^+ \right\}  \gamma + \frac{1}{m}\sum_{j=1}^m \left\{ \sum_{i=1}^n \left(\frac{1}{n} +  \lambda_{i}^+ - \lambda_{i}^- \right) \underset{h_j \sim \, p_j}{\Ps} \left[ h_j (x_i) \neq f_j(x_i) \right] \right\}
\end{align*}

Therefore, the minimization problem of the Learner gets nicely decoupled into $(m+1)$ minimization problems. Let $w_{i} = \lambda_{i}^+ - \lambda_{i}^-$ for all $i$, and acccordingly, let $\wvec = \left[w_{1}, \ldots, w_{n} \right]^\top \in \R^n$. First, the optimal value for $\gamma$ is chosen according to
\begin{equation}\label{eq:gammat}
\gamma = \1 \left[ \sum_{i=1}^n w_{i}  > 0\right]
\end{equation}
And for learning problem $j$, the following minimization problem must be solved.
$$
\argmin_{p_j \, \in \, \Delta(\Hs)} \, \sum_{i=1}^n \left( 1/n + w_{i} \right) \underset{h_j \sim \, p_j}{\Ps} \left[ h_j (x_i) \neq f_j(x_i) \right]  \equiv \argmin_{h_j \, \in \, \Hs} \, \sum_{i=1}^n \left( 1/n + w_{i} \right)  \1 \left[ h_j (x_i) \neq f_j(x_i) \right]
$$
where the equivalence holds since the Learner can choose to put all the probability mass on a single classifier. This minimization problem represents exactly a weighted classification problem. Since we work with cost sensitive classification oracles in this paper, we further reduce the weighted classification problem to a CSC problem that can be solved by a call to the CSC oracle for $\Hs$ ($CSC(\Hs)$). For problem $j \in [m]$, let 
$$
c_{i,j}^1 = (w_{i} + 1/n) (1 - f_j(x_i)) \quad , \quad c_{i,j}^0 = (w_{i} + 1/n) f_j (x_i)
$$
represent the costs associated with individual $i \in [n]$. Observe that the above weighted classification problem can be now casted as the following CSC problem.
\begin{equation}\label{eq:ht}
h_j = \argmin_{h \, \in \, \Hs} \, \sum_{i=1}^n c_{i,j}^1 \, h (x_i) + c_{i,j}^0 \left( 1 - h (x_i) \right)
\end{equation}
To sum up, in each iteration of the algorithm the Auditor first uses the exponentiated gradient descent algorithm to update the dual variable $\lamb$ (or correspondingly, the vector of weights $\wvec$ over the training individuals) and then the Learner picks $\gamma = \1 \left[ \sum_{i=1}^n w_{i}  > 0\right]$ and solves $m$ cost sensitive classification problems casted in (\ref{eq:ht}) by calling the oracle $CSC(\Hs)$ for all $1 \le j \le m $. We have the best response of the Learner written in Subroutine \ref{bestresponse}. This subroutine will be called in each iteration of the final AIF learning algorithm.

\begin{subroutine}
\KwIn{dual weights $\wvec \in \R^n$, training examples $S = \left\{ x_i, \left( f_j(x_i) \right)_{j=1}^m \right\}_{i=1}^n$}
\medskip
$\gamma \leftarrow \1 \left[ \sum_{i=1}^n w_{i} > 0\right]$ \\
\For{$j= 1, \ldots, m$}{
$c_i^1 \leftarrow (w_{i} + 1/n) (1 - f_j(x_i)) $ for $i \in [n]$.\\
$c_i^0 \leftarrow (w_{i} + 1/n) f_j(x_i) $ for $i \in [n]$.\\
 $D \leftarrow \{ x_i , c_i^1, c_i^0 \}_{i=1}^n$\\
$h_{j} \leftarrow CSC \left(\Hs; D \right)$
}
$\hvec \leftarrow \left( h_{1}, \, h_{2}, \, \ldots, \, h_{m} \right)$\\
\medskip
\KwOut{$\left( \hvec , \gamma \right)$}
\caption{\best -- best response of the Learner in the AIF setting}
\label{bestresponse}
\end{subroutine}

\subsection{\textbf{AIF-Learn}: Implementation and In-sample Guarantees}\label{subsec:algorithm}
In Algorithm \ref{algo} (\textbf{AIF-Learn}), with a slight deviation from what we described in the previous subsections, we implement the proposed algorithm. The deviation arises when the Auditor updates the dual variables $\lamb$ in each round, and is introduced in the service of arguing for generalization. To counteract the inherent adaptivity of the algorithm (which makes the quantities estimated at each round data dependent), at each round $t$ of the algorithm, we draw a \emph{fresh} batch of $m_0$ problems to estimate the fairness violation vector $\rvec$ (\ref{eq:r}). From another viewpoint -- which is the way the algorithm is actually implemented -- similar to usual batch learning models we assume we have a training set $F$ of $m$ learning problems upfront. However, in our proposed algorithm that runs for $T$ iterations, we partition $F$ into $T$ equally-sized ($m_0$) subsets $\{F_t\}_{t=1}^T$ uniformly at random and use only the batch of problems $F_t$ at round $t$ to update the dual variables $\lamb$. Without loss of generality and to avoid technical complications, we assume $\vert F_t \vert = m_0 = m/T$ is a natural number. This is represented in Algorithm \ref{algo} by writing $\Qhat_t = \uniform \left( F_t \right)$ for the uniform distribution over the batch of problems $F_t$ , and $\hvec_t \vert_{F_t}$ for the associated learned classifiers for $F_t$. We will see this modification will only introduce an extra $ \Otilde (\sqrt{1/m_0})$ term to the regret of the Auditor and thus we have to assume $m_0$ is sufficiently large (Assumption \ref{ass:m}) so that the Auditor has in fact low regret. 

Notice \textbf{AIF-Learn} takes as input an approximation parameter $\nu \in [0,1]$ which will quantify how close the output of the algorithm is to an equilibrium of the introduced game, and it will accordingly propagate to the accuracy bounds. One important aspect of \textbf{AIF-Learn} is that the algorithm maintains a vector of weights $\wvec_t \in \R^n$ over the training individuals $X$ and that each $\widehat{p}_j$ learned by our algorithm is in fact an average over $T$ classifiers where classifier $t$ is the solution to a CSC problem on $X$ weighted by $\wvec_t$. As a consequence, we propose to extend the learned restricted mapping $\pvechat \in \Delta (\Hs)^m$ to a mapping $\psihat = \psihat \, ( X, \What ) \in \Delta (\Hs)^\F$ that takes \emph{any} problem $f \in \F$ as input (represented to $\psihat$ by the labels it induces on the training individuals), uses $X$ along with the set of weights $\What = \{\wvec_t\}_{t=1}^T$ to solve $T$ CSC problems in a similar fashion, and outputs the average of the learned classifiers denoted by $\psihat_f \in \Delta(\Hs)$. This extension is consistent with $\pvechat$ in the sense that $\psihat$ restricted to $F$ will be exactly the $\pvechat$ output by our algorithm. We have the pseudocode for $\psihat$ written in detail in Mapping \ref{algo:psihat} and we let \textbf{AIF-Learn} output $\psihat$.


\begin{algorithm}
\KwIn{fairness parameter $\alpha$, approximation parameter $\nu$, \\ \ \ \ \ \ \ \ \ \ \ \ training data set $X = \{x_i\}_{i=1}^n$ and $F=\{f_j\}_{j=1}^m$ }
\medskip
\begin{center}
$B \leftarrow \frac{1+ 2 \nu}{ \alpha} , \ \ T \leftarrow  \frac{16 B^2 \left( 1 + 2\alpha  \right)^2  \log \left( 2n+1 \right) }{\nu^2}, \  \ \eta \leftarrow \frac{\nu}{4 \left( 1 + 2 \alpha \right)^2 B}, \ \ m_0 \leftarrow \frac{m}{T}, \ \ S \leftarrow \left\{ x_i, \left( f_j(x_i) \right)_{j=1}^m \right\}_{i=1}^n$
\end{center}
\medskip
Partition $F$ uniformly at random: $F = \{F_t\}_{t=1}^T$ where $| F_t | = m_0$.\\
$\boldsymbol{\theta}_1 \leftarrow \boldsymbol{0} \in \R^{2n}$\\
\For{$t=1, \ldots, T$}{
$\lambda_{i,t}^{\bullet} \leftarrow B \frac{\exp(\theta_{i,t}^{\bullet})}{1 + \sum_{i',\bullet'} \exp(\theta_{i',t}^{\bullet'})}$ for $i \in [n]$ and $\bullet \in \{ +, -\}$ \\
$\wvec_t \leftarrow [ \lambda_{i,t}^+ - \lambda_{i,t}^- ]_{i=1}^n \in \R^n$\\
$( \hvec_t, \gamma_t ) \leftarrow \best (\wvec_t; S)$\\
$\boldsymbol{\theta}_{t+1} \leftarrow \boldsymbol{\theta}_{t} + \eta \cdot \rvec \left(\hvec_t \vert_{F_t}, \gamma_t ; \Qhat_t\right)$
}
$\pvechat \leftarrow \frac{1}{T} \sum_{t=1}^T \hvec_t \ , \quad \gammahat \leftarrow \frac{1}{T} \sum_{t=1}^T \gamma_t \ , \quad \lambhat \leftarrow \frac{1}{T} \sum_{t=1}^T \lamb_t \ , \quad \What \leftarrow \{\wvec_t\}_{t=1}^T$
\medskip

\KwOut{average plays $\left( \pvechat, \, \gammahat, \, \lambhat \right)$, mapping $\psihat = \psihat \left(X, \What \right)$ (see Mapping \ref{algo:psihat})}
\caption{\textbf{AIF-Learn} -- learning subject to AIF}
\label{algo}
\end{algorithm}

\begin{mapping}
\KwIn{$f \in \F$ (represented to $\psihat$ as $\{ f(x_i) \}_{i=1}^n$) }

\medskip
\For{$t=1, \ldots, T$}{
$c_i^1 \leftarrow (w_{i,t} + 1/n) (1 - f(x_i)) $ for $i \in [n]$\\
$c_i^0 \leftarrow (w_{i,t} + 1/n) f(x_i) $ for $i \in [n]$\\
 $D \leftarrow  \{ x_i , c_i^1, c_i^0 \}_{i=1}^n$\\
$h_{f,\wvec_t} \leftarrow CSC \left(\Hs; D \right)$\\
}
$\psihat_{f} \leftarrow \frac{1}{T} \sum_{t=1}^T  h_{f,\wvec_t}$
\medskip

\KwOut{$\psihat_f \in \Delta (\Hs)$}
\caption{{$\boldsymbol{\psihat} \, ( X, \What )$} -- pseudocode for the mapping $\psihat$ output by Algorithm \ref{algo}}
\label{algo:psihat}
\end{mapping}

We start the analysis of Algorithm \ref{algo} by establishing the regret bound of the Auditor over $T$ rounds of the algorithm. The regret bound will help us pick the number of iterations $T$ and the learning rate $\eta$ so that the Auditor has sufficiently small regret (bounded by $\nu$). Notice the Learner uses her best response in each round of the algorithm which implies that she has zero regret. Since in this subsection we eventually want to state in-sample guarantees (i.e., guarantees with respect to the distributions $\Pxhat$ and $\Qhat$), we work with the restricted mapping $ \psihat \vert_F = \pvechat$. In the next subsection we will focus on generalizations in our framework and there we will have to state the guarantees for the mapping $\psihat$. We defer all the proofs to the Appendix
\medskip
\begin{lemma}[Regret of the Auditor]\label{lemma:regretdual}
Let $0< \delta <1$. Let $\{ \lamb_t \}_{t=1}^T$ be the sequence of exponentiated gradient descent plays (with learning rate $\eta$) by the Auditor to the given $\{ \hvec_t , \gamma_t \}_{t=1}^T$ of the Learner over $T$ rounds of Algorithm \ref{algo}. We have that for any set of individuals $X$, with probability at least $1-\delta$ over the problems $F$, the (average) regret of the Auditor is bounded as follows.
\begin{align*}\label{eq:regretlambda}
&\frac{1}{T} \max_{\lamb \in \Lambda} \sum_{t=1}^T \Ls \left(\hvec_t, \gamma_t, \lamb \right) - \frac{1}{T} \sum_{t=1}^T \Ls \left(\hvec_t, \gamma_t, \lamb_t \right) \le B \sqrt{ \frac{\log \left( 2nT / \delta \right)}{2 m_0} }  + \frac{B \log \left(2n +1 \right)}{\eta T} + \eta B \left( 1 + 2 \alpha \right)^2
\end{align*}
\end{lemma}

The last two terms appearing in the above regret bound come from the usual regret analysis of the exponentiated gradient descent algorithm. However, the first term originates from a high probability Chernoff-Hoeffding bound because as explained before, the Auditor is using --- instead of the whole set of problems $F$ --- only $m_0$ randomly selected problem $F_t$ to estimate the vector of fairness violations $\rvec$ at round $t$. Hence at each round $t$, the difference of fairness violation estimates --- one with respect to $F_t$ and another with respect to $F$ --- will appear in the regret of the Auditor which can be bounded by the Chernoff-Hoeffding's inequality. We will therefore have to assume that $m_0$ is sufficiently large to make the above regret bound small enough.
\medskip

\begin{assumption}\label{ass:m}
For a given confidence parameter $\delta$, inputs $\alpha$ and $\nu$ of Algorithm \ref{algo}, we suppose throughout this section that the number of fresh problems $m_0$ used in each round of Algorithm \ref{algo} satisfies $m_0 \ge O \left(\frac{ \log \left( n T / \delta \right)}{\alpha^2 \nu^2} \right)$, or equivalently $m = m_0 \cdot T  \ge O \left(\frac{ T \log \left( n T / \delta \right)}{\alpha^2 \nu^2} \right)$.
\end{assumption}
\medskip

Following Lemma \ref{lemma:regretdual} and Assumption \ref{ass:m}, we characterize the average plays $( \pvechat, \, \gammahat, \, \lambhat )$ output by Algorithm \ref{algo} in the following lemma. Informally speaking, this lemma guarantees that neither player would gain more than $\nu$ if they deviated from these proposed strategies output by the algorithm. This is what we call a \emph{$\nu$-approximate equilibrium} of the game. The proof of the lemma follows from the regret analysis of the Auditor and is fully presented in the Appendix.
\medskip

\begin{lemma}[Average Play Characterization]\label{thm:approxequilib}
Let $0 < \delta < 1$. We have that under Assumption \ref{ass:m}, for any set of individuals $X$, with probability at least $1-\delta$ over the labelings $F$, the average plays $(\pvechat, \gammahat, \lambhat )$ output by Algorithm~\ref{algo} forms a $\nu$-approximate equilibrium of the game, i.e.,
\begin{align*}
&\Ls \left(\pvechat, \gammahat, \lambhat \right) \, \le \, \Ls \left(\pvec, \gamma, \lambhat \right)  + \nu \quad \text{for all  } \ \pvec \in \Delta (\Hs)^m \, , \gamma \in [0,1] \\
&\Ls \left(\pvechat, \gammahat, \lambhat \right) \, \ge \, \Ls \left(\pvechat, \gammahat, \lamb \right) - \nu \quad \text{for all  } \ \lamb \in  \Lambda
\end{align*}
\end{lemma}

We are now ready to present the main theorem of this subsection which takes the guarantees provided in Lemma \ref{thm:approxequilib} and turns them into accuracy and fairness guarantees of the pair $\left(\pvechat, \gammahat \right)$ using the specific form of the Lagrangian (\ref{eq:lagrangian}). The theorem will in fact show that the set of randomized classifiers $\pvechat$ achieves optimal accuracy up to $O \left( \nu \right)$ and that it also satisfies $\left( O \left( \alpha \right) , 0 \right)$-AIF notion of fairness, all with respect to the empirical distributions $\Pxhat$ and $\Qhat$.
\medskip
\begin{theorem}[In-sample Accuracy and Fairness]\label{thm:insampleguarantees}
Let $0< \delta <1$ and suppose Assumption \ref{ass:m} holds. Let $\left(\pvechat, \gammahat \right)$ be the output of Algorithm~\ref{algo} and let $\left( \pvec, \gamma \right)$ be any feasible pair of variables for the empirical fair learning problem (\ref{box:fairerm}). We have that for any set of individuals $X$, with probability at least $1-\delta$ over the labelings $F$,
\begin{align*}
\err  \left(\pvechat; \Pxhat, \Qhat \right) \, \le \, \err \left(\pvec; \Pxhat, \Qhat \right) + 2 \nu
\end{align*}
and that $\pvechat$ satisfies $(3\alpha, 0)$-AIF with respect to the empirical distributions $(\Pxhat, \Qhat)$. In other words, for all $i \in [n]$,
\begin{align*}
\left\vert \Epsilon \left(x_i, \pvechat; \Qhat \right) - \gammahat \right\vert \, \le \, 3\alpha
\end{align*}
\end{theorem} 

%% file: AIF-generalization.tex
\subsection{Generalization Theorems}\label{subsec:generalization}

\begin{figure}
\centering
\includegraphics[scale = 0.5]{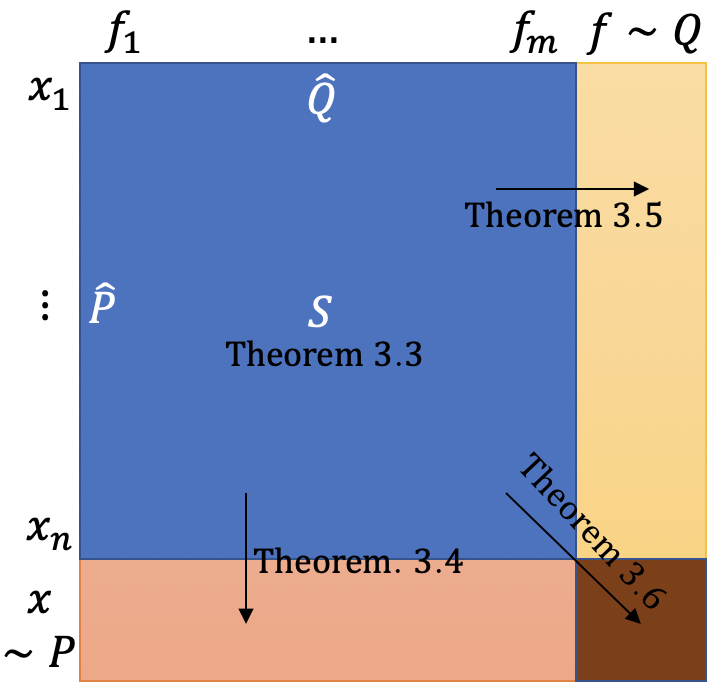}
\caption{Illustration of generalization directions.}
\label{fig:generalization}
\end{figure}

When it comes to out-of-sample performance in our framework, unlike in usual learning settings, there are two distributions we need to reason about: the individual distribution $\Px$ and the problem distribution $\Q$ (see Figure \ref{fig:generalization} for a visual illustration of generalization directions in our framework). We need to argue that $\psihat$ induces a mapping that is accurate with respect to $\Px$ and $\Q$, and is fair for almost every individual $x \sim \Px$, where fairness is defined with respect to the true problem distribution $\Q$. Given these two directions for generalization, we state our generalization guarantees in three steps visualized by arrows in Figure \ref{fig:generalization}. First, in Theorem \ref{thm:generalizationx}, we fix the empirical distribution of the problems $\Qhat$ and show that the output $\psihat$ of Algorithm \ref{algo} is accurate and fair with respect to the underlying individual distribution $\Px$ as long as $n$ is sufficiently large. Second, in Theorem \ref{thm:generalizationf}, we fix the empirical distribution of individuals $\Pxhat$ and consider generalization along the underlying problem generating distribution $\Q$. It will follow from the dynamics of the algorithm, as well as the structure of $\psihat$, that the learned mapping $\psihat$ will remain accurate and fair with respect to $\Q$. We will eventually put these pieces together in Theorem \ref{thm:generalizationxf} and argue that $\psihat$ is accurate and fair with respect to the underlying distributions $\left( \Px, \Q \right)$ simultaniously, given that both $n$ and $m$ are large enough. We will use $\opt$ (see Definition \ref{def:opt}) as a benchmark to evaluate the accuracy of the mapping $\psihat$.
\medskip

\begin{theorem}[Generalization over $\Px$]\label{thm:generalizationx}
Let $0 < \delta < 1$. Let $\left( \psihat, \gammahat \right)$ be the outputs of Algorithm \ref{algo}, and suppose
$$
n \ge \widetilde{O} \left( \frac{m \, d_\Hs + \log \left(1/\nu^2 \delta\right)}{\alpha^2 \beta^2} \right)
$$
where $d_\Hs$ is the VC dimension of $\Hs$. We have that with probability at least $1-5\delta$ over the observed data set $(X,F)$, the mapping $\psihat$ satisfies $\left( 5\alpha, \beta \right)$-AIF with respect to the distributions $\left( \Px, \Qhat \right)$, i.e.,
$$
\underset{x \, \sim \Px}{\Ps} \left( \left\vert \Epsilon \left( x, \psihat; \Qhat \right) - \gammahat \right\vert > 5 \alpha  \right) \le \beta
$$
and that,
$$
\err \left( \psihat ; \Px , \Qhat \right) \le \opt \left( \alpha ; \Px, \Qhat \right) +  O \left( \nu \right) + O \left( \alpha \beta \right)
$$
\end{theorem}
\medskip

The proof of this theorem will use standard VC-type generalization techniques where a Chernoff-Hoeffding bound is followed by a union bound (accompanied with the two-sample trick and Sauer's Lemma) to guarantee a uniform convergence of the empirical estimates to their true expectation. However, compared to the standard VC-based sample complexity bounds in learning theory, there is an extra factor of $m$ because there are $m$ hypotheses to be learned and that the $\alpha^2$ factor appears in the denominator since in our setting a uniform convergence for all -- pure -- classifiers will \emph{not} simply lead to a uniform convergence for all -- randomized -- classifiers without blowing up the sample complexity (specifically when proving generalization for fairness). We will therefore directly prove uniform convergence for randomized classifiers and that our argument will go through by sparsifying the distributions $\pvechat = \psihat \vert_F$ (taking samples of size $\Otilde \left( 1/ \alpha^2 \right)$ from $\pvechat$) coupled with a uniform convergence for $\Otilde \left( 1/ \alpha^2 \right)$-sparse classifiers (randomized classifiers with support size $\le \Otilde \left(1/\alpha^{2} \right)$) and this is how $\alpha^2$ shows up in the sample complexity bound.
\medskip

\begin{theorem}[Generalization over $\Q$]\label{thm:generalizationf}
Let $0 < \delta < 1$. Let $\left( \psihat, \gammahat \right)$ be the outputs of Algorithm \ref{algo} and suppose
$$
m \ge \Otilde \left( \frac{\log \left( n \right) \log \left( n/\delta \right)}{\nu^4 \alpha^4} \right)
$$
We have that for any set of observed individuals $X$, with probability at least $1-6\delta$ over the observed problems $F$, the learned mapping $\psihat$ satisfies $\left( 4\alpha, 0 \right) $-AIF with respect to the distributions $\left( \Pxhat, \Q \right)$, i.e.,
$$
\underset{x \, \sim \Pxhat}{\Ps} \left( \left\vert \Epsilon \left( x, \psihat; \Q \right) - \gammahat \right\vert > 4\alpha  \right) = 0
$$
and that,
$$
\err \left( \psihat ; \Pxhat , \Q \right) \le \opt \left( \alpha ; \Pxhat, \Q \right) + O \left( \nu \right)
$$
\end{theorem}
\medskip

This theorem will follow directly from Chernoff-type concentration inequalities where the fact that in each round the Auditor in our algorithm is using only a fresh batch of randomly selected $m_0$ problems to estimate the fairness violations will help us to prove concentration without appealing to a uniform convergence. The sample complexity for $m$ stated in this theorem is equivalent to that of Assumption \ref{ass:m} because we needed almost the same type of concentration for controlling the regret of the Auditor in the previous subsection. Having proved generalization separately for $\Px$ and $\Q$, we are now ready to state the final theorem of this section which provides generalization guarantees simultaneously over both distributions $\Px$ and $\Q$.

\medskip

\begin{theorem}[Simultaneous Generalization over $\Px$ and $\Q$]\label{thm:generalizationxf}
Let $0 < \delta < 1$. Let $\left( \psihat, \gammahat \right)$ be the outputs of Algorithm \ref{algo} and suppose
$$
n \ge \widetilde{O} \left( \frac{m \, d_\Hs + \log \left(1/\nu^2 \delta\right)}{\alpha^2 \beta^2} \right) \quad , \quad m \ge \Otilde \left( \frac{\log \left( n \right) \log \left( n/\delta \right)}{\nu^4 \alpha^4} \right)
$$
where $d_\Hs$ is the VC dimension of $\Hs$. We have that with probability at least $1-12 \delta$ over the observed data set $(X,F)$, the learned mapping $\psihat $ satisfies $\left( 6\alpha, 2 \beta \right)$-AIF with respect to the distributions $\left( \Px, \Q \right)$, i.e.,
$$
\underset{x \, \sim \Px}{\Ps} \left( \left\vert \Epsilon \left( x, \psihat ; \Q \right) - \gammahat \right\vert > 6 \alpha  \right) \le 2\beta
$$
and that,
$$
\err \left( \psihat ; \Px , \Q \right) \le \opt \left( \alpha ; \Px, \Q \right) + O \left( \nu \right) + O \left( \alpha \beta \right)
$$
\end{theorem}

To prove this theorem we basically start with the guarantees we have for the empirical distributions $ (\Pxhat, \Qhat)$ and lift them into their corresponding guarantees for $( \Px, \Qhat)$ by Theorem \ref{thm:generalizationx}. We will then have to take another ``lifting'' step from $( \Px, \Qhat )$ to $ (\Px, \Q)$ which is not quite similar to what we have shown in Theorem \ref{thm:generalizationf} and will be proved as a separate lemma in the Appendix. Note that the bounds on $n$ and $m$ in Theorem \ref{thm:generalizationxf} are mutually dependent: $n$ must be \emph{linear} in $m$, but $m$ need only be \emph{logarithmic} in $n$, and so both bounds can be simultaneously satisfied with sample complexity that is only polynomial in the parameters of the problem.

%% file: experiment-nips.tex
\begin{figure*}
\centering
	\begin{subfigure}[t]{0.45\textwidth}
	\centering
	\includegraphics[scale = 0.47]{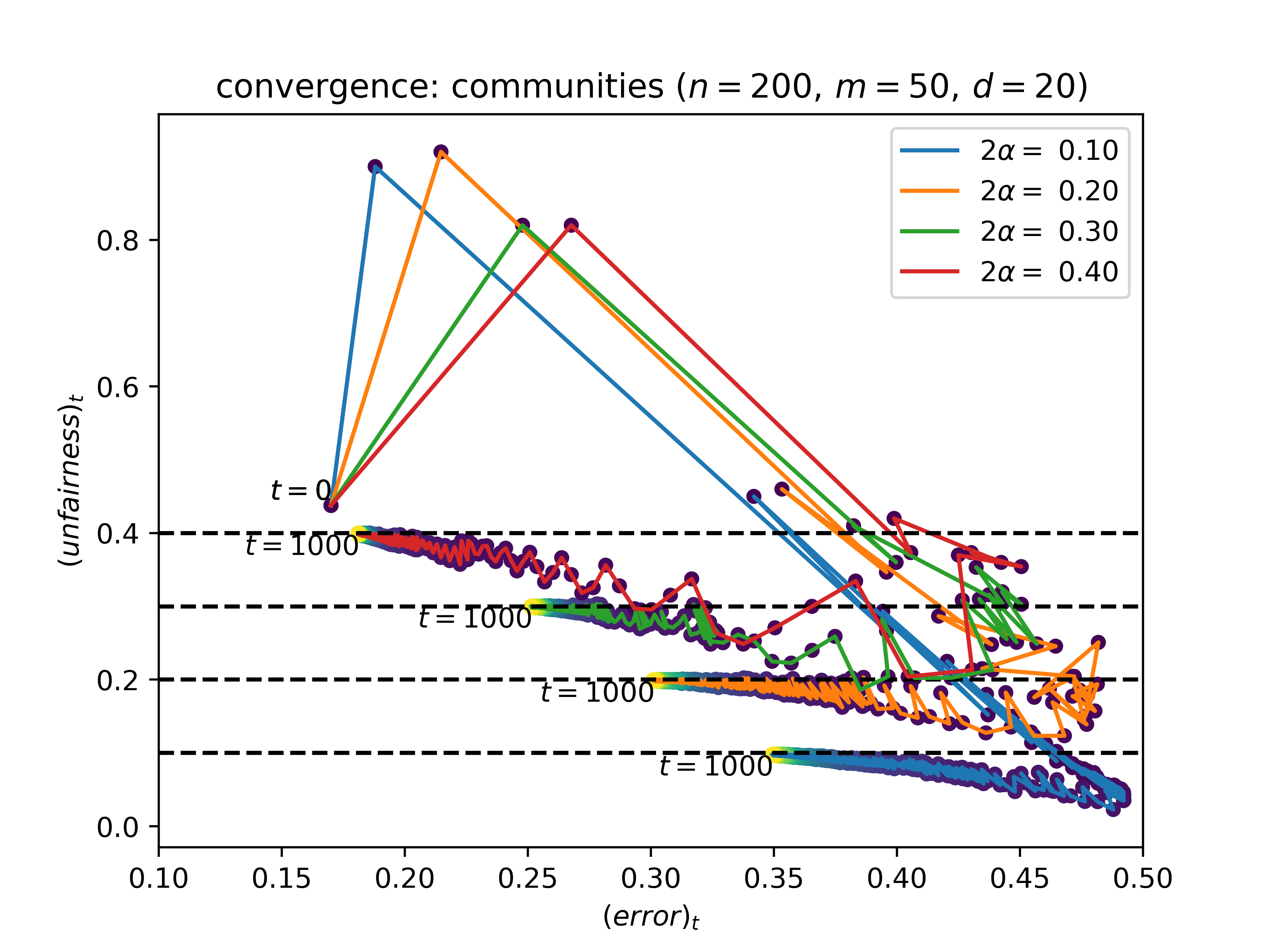}
	\caption{convergence plot: communities data set}
	\end{subfigure}
	\hspace{1cm}
	\begin{subfigure}[t]{0.45\textwidth}
	\centering
	\includegraphics[scale=0.47]{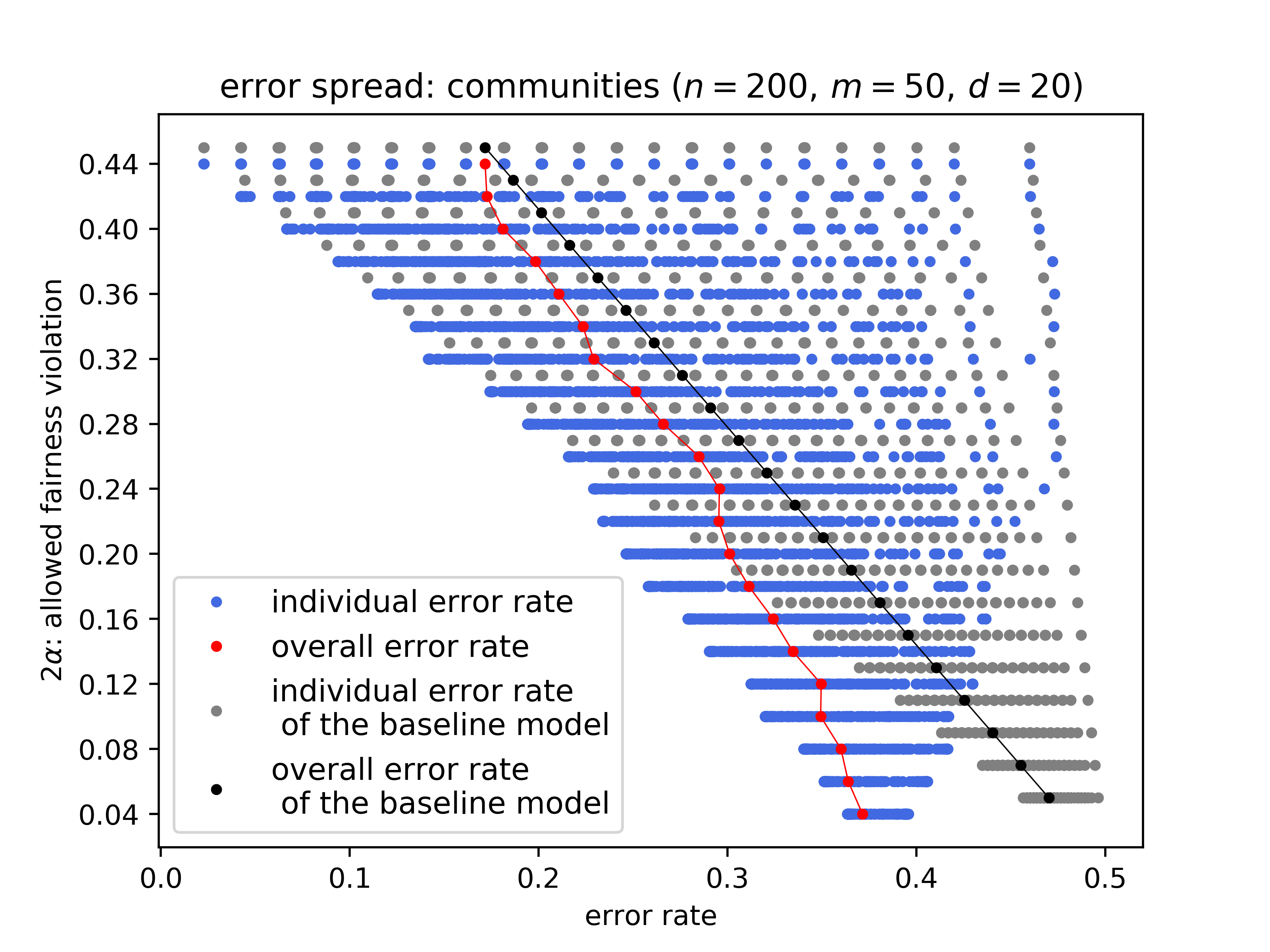}
	\caption{error spread plot: communities data set}
	\end{subfigure}
\caption{
(a) Error-unfairness trajectory plots illustrating the convergence of algorithm \textbf{AIF-Learn}.
(b) Error-unfairness tradeoffs and individual errors for \textbf{AIF-Learn} vs. simple mixtures
	of the error-optimal model and random classification. Gray dots are shifted upwards slightly
	to avoid occlusions.
}
\label{fig}
\end{figure*}

We have implemented
the \textbf{AIF-Learn} algorithm and conclude with a brief experimental demonstration of
its practical efficacy
using the Communities and
Crime dataset\footnote{Described in detail and available for download
at \url{http://archive.ics.uci.edu/ml/datasets/communities+and+crime}}, which contains
U.S. census records with demographic information at the neighborhood level. 
To obtain a challenging instance of our multi-problem framework,
we treated each of the first $n=200$ neighborhoods as the ``individuals'' in our sample, and binarized
versions of the first $m=50$ variables as distinct prediction problems. Another $d=20$ of the 
variables were used as features for learning.
For the base learning oracle assumed by \textbf{AIF-Learn}, we used a linear threshold
learning heuristic that has worked well in other oracle-efficient reductions (\cite{KNRW18}).


Despite the absence of worst-case guarantees for the linear threshold heuristic, 
\textbf{AIF-Learn} seems to empirically enjoy the strong convergence properties suggested by
the theory. In Figure~\ref{fig}(a) we show trajectory plots of the learned model's error ($x$ axis) versus
its fairness violation (variation in cross-problem individual error rates, $y$ axis) over
1000 iterations of the algorithm for varying values of the allowed fairness violation $2\alpha$
(dashed lines). In each case we see the trajectory eventually converges to a point
which saturates the fairness constraint with the optimal error.

In Figure~\ref{fig}(b) we provide a more detailed view of the behavior and performance of \textbf{AIF-Learn}.
The $x$ axis measures error rates, while the $y$ axis measures the allowed fairness violation. For each value of the
allowed fairness violation $2\alpha$ (which is the allowed gap between the smallest and largest individual errors
on input $\alpha$), there is a horizontal row of 200 blue dots showing the error rates for each
individual, and a single red dot representing the overall average of those individual error rates.
As expected, for large $\alpha$ (weak or no fairness constraint), the overall error rate is lowest, but the
spread of individual error rates (unfairness) is greatest. As $\alpha$ is decreased, the spread of individual
error rates is greatly narrowed, at a cost of greater overall error.

A trivial way of achieving zero variability in individual error rates is to make all predictions
randomly. So as a baseline comparison for \textbf{AIF-Learn}, the gray dots in Figure~\ref{fig}(b) show the individual error
rates achieved by different mixtures of the unconstrained error-optimal model with random classifications,
with a black dot representing the overall average of these rates.
When the weight on random classification is low (weak or no fairness, top row of gray dots), the overall error is lowest and
the individual variation (unfairness) is highest. As we increase the weight on random classification,
variation or unfairness decreases and the overall error gets worse. It is clear from the figure that
\textbf{AIF-Learn} is considerably outperforming this baseline, both in terms of the average errors (red vs. black lines)
and the individual errors (blue vs. gray dots).

\begin{figure}
\centering
\includegraphics[scale=0.6]{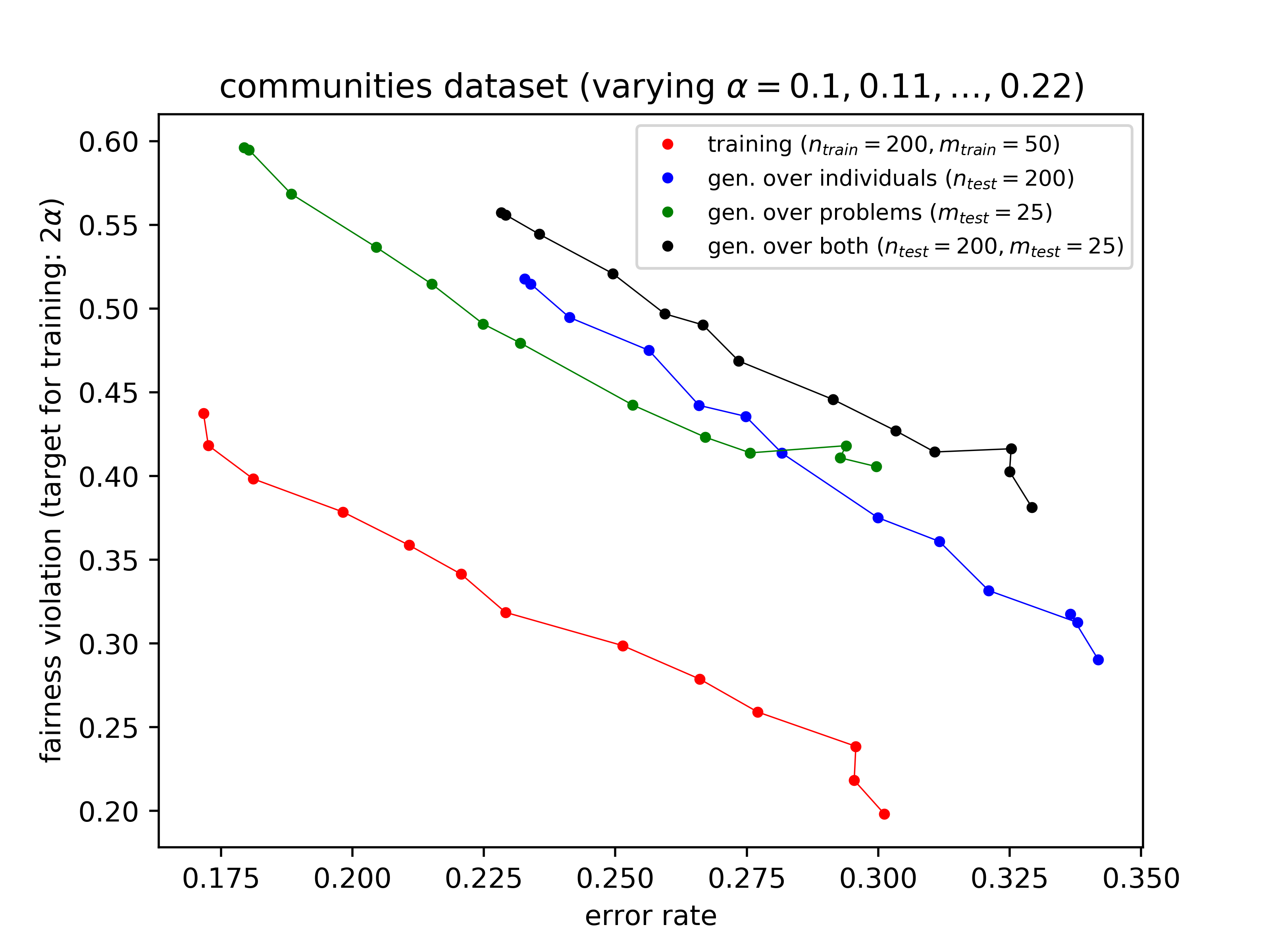}
\caption{Pareto frontier of error and fairness violation rates on training and test data sets.}
\label{fig:gen}
\end{figure}

Finally, we present out-of-sample performance of \textbf{AIF-Learn}, for which we provided theoretical guarantees in Section~\ref{subsec:generalization}, in Figure \ref{fig:gen}. To be consistent with in-sample results reported in Figure \ref{fig}(b), for each value of $\alpha$, we trained a mapping on exactly the same subset of the Communities and Crime data set ($n = 200$ individuals, $m = 50$ problems) that we used before. Thus the red curve labelled ``training'' in Figure \ref{fig:gen} is the same as the red curve appearing in Figure \ref{fig}(b). We used a completely fresh holdout consisting of $n = 200$ individuals and $m = 25$ problems (binarized features from the dataset that weren't previously used) to evaluate our generalization performance over both individuals and problems, in terms of both accuracy and fairness violation. Similar to the presentation of generalization theorems in Section \ref{subsec:generalization}, we demonstrate experimental evaluation of generalization in three steps. The blue and green curves in Figure \ref{fig:gen} represent generalization results over individuals (test data: test individuals and training problems) and problems (test data: training individuals and test problems) respectively. The black curve represent generalization across both individuals and problems where test individuals and test problems were used to evaluate the performance of the trained models. 

Two things stand out from Figure \ref{fig:gen}:
\newline
\begin{enumerate}
\item As predicted by the theory, our test curves track our training curves, but with higher error and unfairness. In particular, the ordering of the models (each corresponds to one $\alpha$) on the Pareto frontier is the same in testing as in training, meaning that the training curve can indeed be used to manage the trade-off out-of-sample as well.
\item The gap in error is substantially smaller than would be predicted by our theory: since our training data set is so small, our theoretical guarantees are vacuous, but all points plotted in our test Pareto curves are non-trivial in terms of both accuracy and fairness. Presumably the gap in error would narrow on larger training data sets.
\end{enumerate}

We present additional experimental results on a synthetic data set in the supplement.

%% file: Appendix.tex

\section{Learning subject to False Positive AIF (FPAIF)}\label{sec:learningfpaif}
\input{FPAIF}

\section{Proofs of Section \ref{sec:learningaif}}\label{app:algorithm}

\subsection{Preliminary Tools}\label{app:prelims}

\begin{theorem}[Additive Chernoff-Hoeffding Bound]\label{thm:chernoff}
Let $X = \{ X_i \}_{i=1}^n$ be a sequence of $i.i.d.$ random variables with $a \le X_i \le b$ and $\E\left[ X_i \right] = \mu$ for all $i$. We have that for all $s > 0$,
$$
\Ps_X \left[ \left\vert \frac{\sum_i X_i}{n} - \mu \right\vert \ge s \right] \le 2 \exp{\left( \frac{-2ns^2}{(b-a)^2} \right)}
$$
\end{theorem}
\medskip

\begin{lemma}[Sauer's Lemma (see e.g. \cite{KearnsV})]\label{lemma:sauer}
Let $\Hs$ be a class of binary functions defined on $\X$ where the VC dimension of $\Hs$, $d_\Hs$, is finite. Let $X=\{x_i\}_{i=1}^n$ be a data set of size $n$ drwan from $\X$ and let $\Hs(X) = \left\{ \left( h(x_1) , \ldots, h(x_n) \right) : h \in \Hs \right\}$ be the set of induced labelings of $\Hs$ on $X$. We have that $\left\vert \Hs (X) \right\vert \le O \left( n^{d_\Hs} \right)$.
\end{lemma}
\medskip

\begin{theorem}[Exponentiated Gradient Descent Regret (see corollary 2.14 of \cite{shalev})]\label{thm:egregret}
Let $\Lambda' = \left\{ \lamb' \in \R^d_+ : \, || \lamb' ||_1 = B \right\}$. Suppose the exponentiated gradient descent algorithm with learning rate $\eta$ is run on the sequence of linear functions $\left\{ f_t (\lamb') = (\lamb')^\top \rvec_t \right\}_{t=1}^T$ where $\lamb' \in \Lambda'$ and $|| \rvec_t ||_\infty \le L$ for all $t$. Let $\lamb'_t$ denote the exponentiated gradient descent play at round $t$. We have that the regret of the algorithm over $T$ rounds is:
$$
\text{Regret}_T (\Lambda') = \max_{\lamb' \in \Lambda'}\sum_{t=1}^T (\lamb')^\top \, \rvec_t - \sum_{t=1}^T (\lamb_t')^\top \, \rvec_t \, \le \, \frac{B \log \left( d \right)}{\eta} + \eta B  L^2 T
$$
and that for $\eta = O \left( 1/ \sqrt{T} \right)$, we have that $\text{Regret}_T (\Lambda') = O \left( \sqrt{T} \right)$.
\end{theorem}
\medskip

\subsection{Algorithm Analysis and In-sample Guarantees}

\begin{proof}[Proof of Lemma~\ref{lemma:regretdual}]
Fix the set of observed individuals $X \in \X^n$. Recall that $\Lambda = \{ \lamb \in \R_+^{2n} : \ ||\lamb||_1 \le B\}$ is the set of strategies for the Auditor. Now let $\Lambda' = \{ \lamb' \in \R_+^{2n+1} : \ ||\lamb'||_1 = B\}$. Any $\lamb \in \Lambda$ is associated with a $\lamb' \in \Lambda'$ which is equal to $\lamb$ on the first $2n$ coordinates and has the remaining mass on the last one. Let $\rhat_t = \rvec \left(\hvec_t \vert_{F_t}, \gamma_t ; \Qhat_t\right)$ be the vector of fairness violations -- estimated over only the $m_0$ problems $F_t$ of round $t$ -- that the Auditor is using in Algorithm \ref{algo}, and let $\rhat'_t \in \R^{2n+1}$ be equal to $\rhat_t$ on the first $2n$ coordinates and zero in the last one. We have that for any $\lamb \in \Lambda$ and its associated $\lamb' \in \Lambda'$, and in particular for $\lamb_t$ and $\lamb_t'$ of Algorithm~\ref{algo}, and all $t \in [T]$,
\begin{equation}\label{eq:4}
\lamb^\top \, \rhat_t = (\lamb')^\top \, \rhat'_t \quad \text{,} \quad \lamb_t^\top \, \rhat_t = (\lamb_t')^\top \, \rhat'_t
\end{equation}
Now by Theorem \ref{thm:egregret}, and using the observation that $|| \rhat_t' ||_\infty = || \rhat_t ||_\infty \le 1 + 2 \alpha$, we have that for any $\lamb' \in \Lambda'$,
\begin{align*}
\sum_{t=1}^T (\lamb')^\top \, \rhat'_t \, \le \, \sum_{t=1}^T (\lamb_t')^\top \, \rhat'_t + \frac{B \log \left(2n +1 \right)}{\eta} + \eta \left( 1 + 2 \alpha \right)^2 B T
\end{align*}
Consequently by Equation~(\ref{eq:4}), we have that for any $\lamb \in \Lambda$,
\begin{equation*}\label{eq:5}
\sum_{t=1}^T \lamb^\top \, \rhat_t \, \le \, \sum_{t=1}^T \lamb_t^\top \, \rhat_t + \frac{B \log \left(2n +1 \right)}{\eta} + \eta \left( 1 + 2 \alpha \right)^2 B T
\end{equation*}
Now let $\rvec_t = \rvec \left(\hvec_t , \gamma_t ; \Qhat \right)$ be the vector of fairness violations estimated over all problems $F$ and notice that the regret bound must be with respect to $\rvec_t$ and not $\rhat_t$. With that goal in mind, we have that for any $\lamb \in \Lambda$,
$$
\sum_{t=1}^T \lamb^\top \, \rvec_t \, \le \, \sum_{t=1}^T \lamb_t^\top \, \rvec_t + \sum_{t=1}^T \left( \lamb_t - \lamb \right)^\top \left( \rhat_t - \rvec_t \right) + \frac{B \log \left(2n +1 \right)}{\eta} + \eta \left( 1 + 2 \alpha \right)^2 B T
$$
We will use Chernoff-Hoeffding's inequality to bound the difference $ \left( \rhat_t - \rvec_t \right) $ in $\ell_\infty$ norm. Let $\psihat_t = \psihat \left(X, \wvec_t \right)$ (see Mapping \ref{algo:psihat}) where $\wvec_t$ is the vector of weights used in round $t$ of the algorithm and observe that we can rewrite $\rhat_t$ and $\rvec_t $ in terms of $\psihat_t$:
$$
\rhat_t = \begin{bmatrix} \Epsilon \left(x_i, \psihat_t ; \Qhat_t \right) - \gamma_t - 2\alpha \\ \gamma_t - \Epsilon \left(x_i, \psihat_t ; \Qhat_t \right) - 2 \alpha \end{bmatrix}_{i=1}^n \quad ,  \quad \rvec_t =  \begin{bmatrix} \Epsilon \left(x_i, \psihat_t ; \Qhat  \right) - \gamma_t - 2\alpha \\ \gamma_t - \Epsilon \left(x_i, \psihat_t ; \Qhat  \right) - 2 \alpha \end{bmatrix}_{i=1}^n
$$
Hence, bounding the difference $ \left( \rhat_t - \rvec_t \right) $ in $\ell_\infty$ norm involves bounding the terms
$$
\left\vert \Epsilon \left(x_i, \psihat_t ; \Qhat_t \right) - \Epsilon \left(x_i, \psihat_t ; \Qhat \right) \right\vert
$$
for all $i$. Notice we can now view the batch of problems $F_t$ as independent draws from the distribution $\Qhat$. Therefore, it follows from the Chernoff-Hoeffding's Theorem \ref{thm:chernoff} that with probability at least $1-\delta$ over the set of problems $F$, for any $\lamb \in \Lambda$,
\begin{align*}
&\sum_{t=1}^T \left( \lamb_t - \lamb \right)^\top \left( \rhat_t - \rvec_t \right)  \le \sum_{t=1}^T || \lamb_t - \lamb ||_1 \cdot ||  \rhat_t - \rvec_t  ||_\infty \le B \sum_{t=1}^T ||  \rhat_t - \rvec_t  ||_\infty \le BT  \sqrt{ \frac{\log \left( 2nT / \delta \right)}{2 m_0} }
\end{align*}
which implies with probability at least $1-\delta$ over the problems $F$, for any $\lamb \in \Lambda$,
\begin{align*}
&\frac{1}{T} \sum_{t=1}^T \Ls \left(\hvec_t, \gamma_t, \lamb \right) - \frac{1}{T} \sum_{t=1}^T \Ls \left(\hvec_t, \gamma_t, \lamb_t \right) \, \le \, B \sqrt{ \frac{\log \left( 2nT / \delta \right)}{2 m_0} }  + \frac{B \log \left(2n +1 \right)}{\eta T} + \eta \left( 1 + 2 \alpha \right)^2 B
\end{align*}
completing the proof.
\end{proof}
\medskip

\begin{proof}[Proof of Lemma~\ref{thm:approxequilib}]
Let
$$
R_{\lamb} := B \sqrt{ \frac{\log \left( 2nT / \delta \right)}{2 m_0} }  + \frac{B \log \left(2n +1 \right)}{\eta T} + \eta \left( 1 + 2 \alpha \right)^2 B
$$
be the average regret of the Auditor. We have that for any $\pvec \in \Delta (\Hs)^m$ and $\gamma \in [0,1]$,
\begin{align*}
\Ls \left(\pvec, \gamma, \lambhat \right) &= \frac{1}{T} \sum_{t=1}^T \Ls \left(\pvec, \gamma, \lamb_t \right)   \quad \text{(by linearity of }\Ls) \\
&\ge \frac{1}{T} \sum_{t=1}^T \Ls \left(\hvec_t, \gamma_t, \lamb_t \right) \quad ( (\hvec_t, \gamma_t ) \text{ is Learner's Best Response})\\
&\ge \frac{1}{T} \sum_{t=1}^T \Ls \left(\hvec_t, \gamma_t, \lambhat \right) - R_{\lamb} \quad (\text{w.p. $1-\delta$ over $F$ by Lemma~\ref{lemma:regretdual}}) \\
&= \Ls \left(\pvechat, \gammahat, \lambhat \right) - R_{\lamb}
\end{align*}
And that for any $\lamb \in \Lambda$ we have:
\begin{align*}
\Ls \left(\pvechat, \gammahat, \lamb \right) &= \frac{1}{T} \sum_{t=1}^T \Ls \left(\hvec_t, \gamma_t, \lamb \right)  \quad \text{(by linearity of }\Ls)\\
&\le \frac{1}{T} \sum_{t=1}^T \Ls \left(\hvec_t, \gamma_t, \lamb_t \right) + R_{\lamb} \quad (\text{w.p. $1-\delta$ over $F$ by Lemma~\ref{lemma:regretdual}}) \\
& \le \frac{1}{T} \sum_{t=1}^T \Ls \left(\pvechat, \gammahat, \lamb_t \right) + R_{\lamb} \quad ( (\hvec_t, \gamma_t ) \text{ is Learner's Best Response})\\
&= \Ls \left(\pvechat, \gammahat, \lambhat \right) + R_{\lamb}
\end{align*}
Now we have to pick $T$ and $\eta$ of Algorithm \ref{algo} such that $R_{\lamb} \le \nu$ where $\nu$ is the approximation parameter input to the algorithm. First let $B = \frac{1+2\nu}{\alpha}$ and observe that if $m_0 \ge \frac{2 B^2 \log \left( 2nT/\delta \right)}{\nu^2} = \frac{2 \left( 1+2\nu\right)^2 \log \left( 2nT/\delta \right)}{\nu^2 \alpha^2} $ (see Assumption \ref{ass:m} on $m_0$) we have that
$$
R_{\lamb} \le \frac{\nu}{2} + \frac{B \log \left(2n +1 \right)}{\eta T} + \eta \left( 1 + 2 \alpha \right)^2 B
$$
Next, let $T = \frac{16 B^2 \left( 1 + 2\alpha  \right)^2 \log \left( 2n+1 \right) }{\nu^2}$, and $\eta = \frac{\nu}{4 \left(1+2\alpha\right)^2 B}$ which makes the last two terms appearing in the RHS of the inequality to sum to $\nu/2$, and accordingly $R_{\lamb} \le \nu$. Therefore, we have shown that for any $X$, with probability at least $1-\delta$ over the observed labelings $F$,
\begin{align*}
&\Ls \left(\pvechat, \gammahat, \lambhat \right) \, \le \, \Ls \left(\pvec, \gamma, \lambhat \right)  + \nu \quad \text{for all } \pvec \in \Delta (\Hs)^m \, , \gamma \in [0,1] \\
&\Ls \left(\pvechat, \gammahat, \lambhat \right) \, \ge \, \Ls \left(\pvechat, \gammahat, \lamb \right) - \nu \quad \text{for all } \lamb \in  \Lambda
\end{align*}
\end{proof}

\begin{proof}[Proof of Theorem~\ref{thm:insampleguarantees}]
Let $\left( \pvec , \gamma \right)$ be any feasible point of the constrained optimization problem (\ref{box:fairerm}) (note as discussed in the body of the paper there is at least one) and define $\lamb^\star \in \Lambda$ to be
$$
\lamb^\star := \begin{cases}
0 & \text{if } r_{k^\star} \left( \pvechat, \gammahat; \Qhat \right) \le 0
\\
B e_{k^\star} & \text{if } r_{k^\star} \left( \pvechat, \gammahat; \Qhat \right) > 0
\end{cases}
$$
where $k^\star = \argmax_{1 \le k \le 2n} r_{k} \left( \pvechat, \gammahat ; \Qhat\right)$.
Observe that with probability $1-\delta$ over $F$,
\begin{align*}
\Ls \left(\pvechat, \gammahat, \lambhat \right) &\le \Ls \left(\pvec, \gamma, \lambhat \right) + \nu \quad (\text{by Lemma \ref{thm:approxequilib}})\\
&= \err \left( \pvec ; \Pxhat, \Qhat \right) + \lambhat^\top \rvec \left( \pvec , \gamma; \Qhat \right) + \nu\\
&\le \err \left( \pvec ; \Pxhat, \Qhat \right) + \nu
\end{align*}
and that,
\begin{align*}
\Ls \left(\pvechat, \gammahat, \lambhat \right) &\ge \Ls \left(\pvechat, \gammahat, \lamb^\star \right) - \nu  \quad (\text{by Lemma \ref{thm:approxequilib}}) \\
&= \err \left( \pvechat; \Pxhat, \Qhat \right) + (\lamb^\star)^\top \rvec \left( \pvechat , \gammahat ; \Qhat \right) - \nu \\
&\ge \err \left( \pvechat ; \Pxhat, \Qhat \right) - \nu
\end{align*}
Combining the above upper and lower bounds on $\Ls \left(\pvechat, \gammahat, \lambhat \right)$ implies
$$
\err \left( \pvechat ; \Pxhat, \Qhat \right) \le \err \left( \pvec ; \Pxhat, \Qhat \right) + 2 \nu
$$
Now let's prove the bound on fairness violation. Once again using the above upper and lower bounds, we have that
$$
(\lamb^\star)^\top \rvec \left( \pvechat , \gammahat; \Qhat \right) \le  \err \left( \pvec ; \Pxhat, \Qhat \right) - \err \left( \pvechat ; \Pxhat, \Qhat \right) + 2 \nu \le 1 + 2 \nu
$$
By definition of $\lamb^\star$,
$$
\max_{1\le k \le 2n} r_{k} \left( \pvechat, \gammahat ; \Qhat \right) \le \frac{1+2 \nu}{B}
$$
which implies for all $1 \le i \le n$,
$$
\left\vert \Epsilon \left(x_i, \pvechat ; \Qhat \right) - \gammahat \right\vert \le 2 \alpha + \frac{1+2 \nu}{B}
$$
And the proof is completed by the choice of $B = \left(1 + 2 \nu \right)/\alpha$ in Algorithm \ref{algo}.
\end{proof}

\subsection{Proof of Theorem \ref{thm:generalizationx}: Generalization over $\Px$}
When arguing about generalization over the distribution $\Px$, we will have to come up with a \emph{uniform convergence} for all \emph{randomized} classifiers. It is usually the case in learning theory -- for example when considering the concentration of classifiers' errors --  that a uniform convergence for pure classifiers will immediately imply the uniform convergence for randomized classifiers without blowing up the sample complexity. However, in our setting and in particular for our notion of fairness, we will have to directly argue about the uniform convergence of randomized classifiers. Although there are possibly infinitely many of those randomized classifiers, we actually need to consider only the $\Otilde \left(1/ \alpha^2 \right)$-sparse classifiers by which we mean the distributions over $\Hs$ with support of size at most $\Otilde \left(1/ \alpha^2 \right)$. This $1/\alpha^2$ factor will accordingly show up in our final sample complexity bound for $n$.
\medskip

\begin{definition}[Sparse Randomized Classifiers]\label{def:rsparse}
We say a randomized classifier $p \in \Delta \left( \Hs \right)$ is $r$-sparse if the support of $p$ is of size at most $r$. We denote the set of all $r$-sparse randomized classifiers of the hypothesis class $\Hs$ by $\Delta_r ( \Hs )$, and the elements of $\Delta_r ( \Hs )$ by $\pr$.
\end{definition}
\medskip

\begin{lemma}[Uniform Convergence of AIF Fairness Notion in $\Delta_r (\Hs)^m$]\label{lemma:s1}
Let $0 < \delta < 1$ and $r \in \N$ and let $T$ be the number of iterations in Algorithm \ref{algo}. Suppose
$$
n \ge \widetilde{O} \left( \frac{r \, m \, d_\Hs + \log \left(T/  \delta\right)}{ \beta^2} \right)
$$
We have that for any set of problems $F$, with probability at least $1-\delta$ over the individuals $X$: for all $\pvecr \in \Delta_r ( \Hs )^m$ and all $\gamma \in  \left\{ 0, \frac{1}{T}, \frac{2}{T}, \ldots, 1 \right\}$,
$$
\left\vert  \underset{x \, \sim \Px}{\Ps} \left( \left\vert \Epsilon \left( x, \pvecr; \Qhat \right) - \gamma \right\vert >  \alpha \right) - \underset{x \, \sim \Pxhat}{\Ps} \left( \left\vert \Epsilon \left( x, \pvecr; \Qhat \right) - \gamma \right\vert >  \alpha \right) \right\vert \le \beta
$$
\end{lemma}

\begin{proof}[Proof of Lemma \ref{lemma:s1}]
The proof of this Lemma will use standard techniques for proving uniform convergence in learning theory such as the ``two-sample trick" and Sauer's Lemma \ref{lemma:sauer}. To simplify notation, let's call, for any $\pvecr \in \Delta_r ( \Hs )^m$ and any $\gamma \in S_T := \left\{ 0, \frac{1}{T}, \frac{2}{T}, \ldots, 1 \right\}$,
\begin{equation}\label{eq:g}
\underset{x \, \sim \Px}{\Ps} \left( \left\vert \Epsilon \left( x, \pvecr; \Qhat \right) - \gamma \right\vert >  \alpha \right) := g \left( \pvecr, \gamma ; \Px \right)
\end{equation}
Define event $A \left(X \right)$ as follows:
$$
A \left(X \right) = \left\{ \exists \, \pvecr \in \Delta_r ( \Hs )^m, \gamma \in S_T :  \left\vert g \left( \pvecr, \gamma ; \Px \right) - g \left( \pvecr, \gamma ; \Pxhat_X \right) \right\vert > \beta \right\}
$$
where $\Pxhat_X \equiv \Pxhat$ represents the uniform distribution over $X$. Our ultimate goal is to show that given the sample complexity for $n$ stated in the lemma, $\Ps_X \left[A \left(X \right) \right]$ is small. Suppose besides the original data set $X$, we also have an ``imaginary" data set of individuals $X' = \{x_i'\}_{i=1}^n \in \X^n$ sampled $i.i.d.$ from the distribution $\Px$. Define event $B \left(X, X' \right)$ as follows:
$$
B \left(X,X' \right) = \left\{ \exists \, \pvecr \in \Delta_r ( \Hs )^m, \gamma \in S_T :  \left\vert g \left( \pvecr, \gamma ; \Pxhat_{X'} \right) - g \left( \pvecr, \gamma ; \Pxhat_X \right) \right\vert > \frac{\beta}{2} \right\}
$$
\textbf{Claim}: $\Ps_X \left[A \left(X \right) \right] \le 2 \, \Ps_{(X,X')} \left[ B \left(X,X' \right) \right]$. Proof: It suffices to show that 
$$\Ps_{(X,X')} \left[ B(X,X') \, \vert \,  A(X) \right] \ge 1/2$$ because $$ \Ps_{(X,X')} \left[ B \left(X,X' \right) \right] \ge \Ps_{(X,X')} \left[ B(X,X') \, \vert \, A(X) \right]  \, \Ps_X \left[A \left(X \right) \right] $$
Let the event $A(X)$ hold and suppose the pair $\pvecr_\star$ and $\gamma_\star $ satisfy
$$\left\vert g \left( \pvecr_\star, \gamma_\star ; \Px \right) - g \left( \pvecr_\star, \gamma_\star ; \Pxhat_X \right) \right\vert > \beta$$
We have that by the triangle inequality and a Chernoff-Hoeffding's bound:
\begin{align*}
&\Ps_{(X,X')} \left[ B(X,X') \, \vert \, A(X) \right] \\
&\ge \Ps_{(X,X')} \left[ \left\vert g \left( \pvecr_\star, \gamma_\star ; \Pxhat_{X'} \right) - g \left( \pvecr_\star, \gamma_\star ; \Pxhat_X \right) \right\vert > \frac{\beta}{2} \right] \\
&\ge \Ps_{(X,X')} \left[ \left\vert g \left( \pvecr_\star, \gamma_\star ; \Px \right) - g \left( \pvecr_\star, \gamma_\star ; \Pxhat_X \right) \right\vert - \left\vert g \left( \pvecr_\star, \gamma_\star ; \Px \right) - g \left( \pvecr_\star, \gamma_\star ; \Pxhat_{X'} \right) \right\vert  > \frac{\beta}{2} \right] \\
&\ge  \Ps_{X'} \left[ \left\vert g \left( \pvecr_\star, \gamma_\star ; \Px \right) - g \left( \pvecr_\star, \gamma_\star ; \Pxhat_{X'} \right) \right\vert < \frac{\beta}{2} \right] \\
&\ge 1 - 2 e^{-n \frac{\beta^2}{2}} \\
&\ge 1/2
\end{align*}
Following the claim, it now suffices to show that $\Ps_{(X,X')} \left[ B(X,X') \right]$ is small. Consider the following thought experiment: Let $T$ and $T'$ be two empty sets. For each $i \in [n]$ toss a fair coin independently and
\begin{itemize}
\item if it lands on Heads, put $x_i$ in $T$ and $x_i'$ in $T'$.
\item if it lands on Tails, put $x_i'$ in $T$ and $x_i$ in $T'$.
\end{itemize}
We will later denote  the randomness induced by these coin flips by ``coin" in our probability statements. It follows immediately by our construction that the distribution of $\left(T,T' \right)$ is the same as the distribution of $\left(X,X' \right)$, and therefore
\begin{equation}\label{eq:hey}
\Ps_{(X,X')} \left[ B(X,X') \right] = \Ps_{(T,T')} \left[ B(T,T') \right] = \Ps_{(X,X', \text{coin})} \left[ B(T,T') \right]
\end{equation}
where
$$
B(T,T') = \left\{ \exists \, \pvecr \in \Delta_r ( \Hs )^m , \gamma \in S_T:  \left\vert g \left( \pvecr, \gamma ; \Pxhat_{T'} \right) - g \left( \pvecr, \gamma ; \Pxhat_T \right) \right\vert > \frac{\beta}{2} \right\}
$$
But we have that
\begin{equation}\label{eq:hey2}
\Ps_{(X,X', \text{coin})} \left[ B(T,T') \right] = \E_{(X,X')} \left[ \Ps_{\text{coin}} \left[ B(T,T') \right] \right]
\end{equation}
where we use the fact that the coin flips are independent of the random variables $(X,X')$ and thus, conditioning on $(X,X')$ won't change the distribution of ``coin". Following Equations (\ref{eq:hey}) and (\ref{eq:hey2}), it now suffices to show that for -- any $(X,X')$ -- $\Ps_{\text{coin}} \left[ B(T,T')  \right]$ is small. Fix the data sets $(X,X')$. Fix $\pvecr \in \Delta_r ( \Hs )^m$ and $\gamma \in S_T$ and let
$$
I = \left\{ i \in [n]: g \left( \pvecr, \gamma ; x_i \right) \neq g \left( \pvecr, \gamma ; x'_i \right) \right\}
$$
where recall that by Equation (\ref{eq:g}):
$$
g \left( \pvecr, \gamma ; x_i \right) = \1 \left[ \left\vert \Epsilon \left( x_i, \pvecr; \Qhat \right) - \gamma \right\vert >  \alpha \right]
$$
Let $\left\vert I \right\vert = n' \le n$. Observe that
\begin{align*}
\Ps_{\text{coin}} \left[ \left\vert g \left( \pvecr, \gamma ; \Pxhat_{T'} \right) - g \left( \pvecr, \gamma ; \Pxhat_T \right) \right\vert > \frac{\beta}{2}  \right] = \Ps_{\text{coin}} \left[ \left\vert \text{heads}(I) - \text{tails}(I) \right\vert > \frac{n\beta}{2}  \right]
\end{align*}
where $ \text{heads}(I)$ and $\text{tails}(I)$ denote the number of heads and tails of the coin on indices $I$. But a triangle inequality followed by two Chernoff-Hoeffding bounds imply
\begin{align*}
\Ps_{\text{coin}} \left[ \left\vert \text{heads}(I) - \text{tails}(I) \right\vert > \frac{n\beta}{2}  \right] &\le \Ps_{\text{coin}} \left[ \left\vert \frac{\text{heads}(I)}{n'} - \frac{1}{2} \right\vert > \frac{n\beta}{4n'}  \right] \\
&\ \ \ + \Ps_{\text{coin}} \left[ \left\vert \frac{1}{2} - \frac{\text{tails}(I)}{n'} \right\vert > \frac{n\beta}{4n'}  \right] \\
&\le 4 e^{-2 n' (n\beta/4n')^2} \\
&\le 4 e^{-n \beta^2 /8}
\end{align*}
Therefore, we have proved that for any $\pvecr \in \Delta_r ( \Hs )^m$ and any $\gamma \in S_T$,
$$
\Ps_{\text{coin}} \left[ \left\vert g \left( \pvecr, \gamma ; \Pxhat_{T'} \right) - g \left( \pvecr, \gamma ; \Pxhat_T \right) \right\vert > \frac{\beta}{2}  \right] \le 4 e^{-n \beta^2 /8}
$$
Now it's time to apply Sauer's Lemma \ref{lemma:sauer} to get a uniform convergence for all $\pvecr \in \Delta_r ( \Hs )^m$ and all $\gamma \in S_T$. Notice once the data sets $(X,X')$ are fixed, there are at most $O \left( \left( 2 n \right)^{d_\Hs \, r \, m}\right)$ number of randomized classifiers $\pvecr \in \Delta_r ( \Hs )^m$ induced on the set $\{X,X'\}$. Here $2n$ comes from the fact that $X$ and $X'$ have a combined number of $2n$ points. $d_\Hs$ is the VC dimension, and $r$ and $m$ show up in the bound because each $\pr_j$ in $\pvecr$ is $r$-sparse and that there are a total of $m$ randomized classifiers $\pr_j$ in $\pvecr$. We also have that $\vert S_T \vert = T + 1$. Consequently by a union bound over all induced $\pvecr$ and all $\gamma \in S_T$:
\begin{align*}
 \Ps_{\text{coin}} \left[ B(T,T') \right] &= \Ps_{\text{coin}} \left[ \exists \, \pvecr \in \Delta_r ( \Hs )^m, \gamma \in S_T :  \left\vert g \left( \pvecr, \gamma ; \Pxhat_{T'} \right) - g \left( \pvecr, \gamma ; \Pxhat_T \right) \right\vert > \frac{\beta}{2} \right] \\
&\le \sum_{\pvecr, \gamma} \Ps_{\text{coin}} \left[ \left\vert g \left( \pvecr, \gamma ; \Pxhat_{T'} \right) - g \left( \pvecr, \gamma ; \Pxhat_T \right) \right\vert > \frac{\beta}{2}  \right] \\
&\le O \left( \left( 2 n \right)^{d_\Hs \, r \, m}\right) \cdot (T+1) \cdot 4 e^{-n \beta^2 /8}
\end{align*}
where the second sum is actually over all the induced $r$-sparse randomized classifiers on the set $\{X,X'\}$ and all $\gamma \in S_T$. This shows so long as
$$
n \ge \widetilde{O} \left( \frac{r \, m \, d_\Hs + \log \left(T/ \delta\right)}{ \beta^2} \right)
$$
we have that for any $(X,X')$, $  \Ps_{\text{coin}} \left[ B(T,T') \right] \le \delta/2$. Therefore
$$
\Ps_X \left[A(X) \right] \le 2 \, \Ps_{(X,X')} \left[ B \left(X,X' \right) \right] = 2 \, \E_{(X,X')} \left[ \Ps_{\text{coin}} \left[ B(T,T') \right] \right] \le \delta
$$
completing the proof.
\end{proof}

\begin{proof}[Proof of Theorem \ref{thm:generalizationx}]
We will use Lemma \ref{lemma:s1} in the proof of this theorem. We are interested in the fairness violation and accuracy of the mapping $\psihat = \psihat \left(X,F\right)$ returned by Algorithm \ref{algo} with respect to the distributions $\Px$ and  $\Qhat$. Fix the set of problems $F$ and observe that when we work with the empirical distribution of the problems $\Qhat$, we have that
$$
\Epsilon \left( x, \psihat; \Qhat \right) = \Epsilon \left( x, \pvechat; \Qhat \right) \quad , \quad \err \left( \psihat; \Px, \Qhat \right) = \err \left( \pvechat; \Px, \Qhat \right)
$$
where $\pvechat$ is the set of $m$ randomized classifiers output by Algorithm \ref{algo}. Let's first prove the generalization for fairness. Define event $A \left( X \right)$:
$$
A \left( X \right) = \left\{   \underset{x \, \sim \Px}{\Ps} \left( \left\vert \Epsilon \left( x, \pvechat; \Qhat \right) - \gammahat \right\vert >  5\alpha \right) - \underset{x \, \sim \Pxhat}{\Ps} \left( \left\vert \Epsilon \left( x, \pvechat; \Qhat \right) - \gammahat \right\vert >  3\alpha \right)  > \beta \right\}
$$
We will eventually show that under the stated sample complexity for $n$ in the theorem, $\Ps_X \left[ A \left( X \right) \right]$ is small. Consider the distributions $\pvechat = \left( \phat_1, \phat_2, \ldots, \phat_m \right)$ over $\Hs$. Let 
$$
r = \frac{\log \left( 12nm/\delta\right)}{ 2 \alpha^2}
$$
For each $j \in [m]$, consider drawing $r$ independent samples from the distribution $\phat_j$ and define $\prhat_j$ to be the uniform distribution over the drawn samples. We will abuse notation and use $\prhat_j$ to denote both the drawn samples and the uniform distribution over them. Now define
$$
\pvechatr = \left( \prhat_1, \prhat_2, \ldots, \prhat_m \right) \in \Delta_r (\Hs)^m
$$
which is the ``$r$-sparsified" version $\pvechat$. One important observation that follows from the Chernoff-Hoeffding's theorem is that for -- any -- $X' = \{ x_i' \}_{i=1}^n \in \X^n$, with probability at least $1-\delta/6$ over the draws of $\pvechatr$, we have that for all $i \in [n]$:
\begin{equation}
\left\vert \Epsilon \left( x_i',\pvechat ; \Qhat \right) - \Epsilon \left( x_i',\pvechat^{(r)}; \Qhat \right) \right\vert \le \alpha
\end{equation}
In other words, for any set of individuals $X' \in \X^n$ of size $n$, we have that:
\begin{equation}\label{eq:hey3}
\Ps_{\pvechatr} \left[ \underset{x \sim \Pxhat'}{\Ps} \left( \left\vert \Epsilon \left( x,\pvechat; \Qhat \right) - \Epsilon \left( x,\pvechat^{(r)}; \Qhat \right) \right\vert > \alpha \right) \neq 0 \right] \le \delta/6
\end{equation}
where $\Pxhat'$ denotes the uniform distribution over $X'$. Now define events $B,C,D$ as follows:
$$
B \left( X, \pvechatr \right) = \left\{ \left\vert \underset{x \, \sim \Px}{\Ps} \left( \left\vert \Epsilon \left( x, \pvechatr; \Qhat \right) - \gammahat \right\vert >  4\alpha \right) - \underset{x \, \sim \Pxhat}{\Ps} \left( \left\vert \Epsilon \left( x, \pvechatr; \Qhat \right) - \gammahat \right\vert >  4\alpha \right) \right\vert > \frac{\beta}{2}  \right\}
$$
$$
C \left( X, \pvechatr \right) = \left\{ \underset{x \sim \Px}{\Ps} \left( \left\vert \Epsilon \left( x,\pvechat ; \Qhat \right) - \Epsilon \left( x,\pvechat^{(r)}; \Qhat \right) \right\vert > \alpha \right) > \frac{\beta}{2} \right\}
$$
$$
D \left( X, \pvechatr \right) = \left\{ \underset{x \sim \Pxhat}{\Ps} \left( \left\vert \Epsilon \left( x,\pvechat; \Qhat \right) - \Epsilon \left( x,\pvechat^{(r)}; \Qhat \right) \right\vert > \alpha \right) \neq 0 \right\}
$$
It follows by the triangle inequality that
\begin{align*}
\Ps_X \left[ A \left( X \right) \right] &= \Ps_{X,\pvechatr} \left[ A \left( X \right) \right] \\
&\le \underbrace{\Ps_{X,\pvechatr} \left[ B \left( X, \pvechatr \right) \right]}_\text{term 1} + \underbrace{\Ps_{X,\pvechatr} \left[ C \left( X, \pvechatr \right) \right]}_\text{term 2} + \underbrace{\Ps_{X,\pvechatr} \left[ D \left( X, \pvechatr \right) \right]}_\text{term 3}
\end{align*}
So to prove that $\Ps_X \left[ A \left( X \right) \right]$ is small, it suffices to show that all the three terms appearing in the RHS of the above inequality are small. We will in fact show each term $\le \delta/3$:
\begin{itemize}
\item term 1: Notice to bound this term we need to prove a \emph{uniform convergence} for all $r$-sparse set of randomized classifiers $\pvecr \in \Delta_r (\Hs)^m$ and all $\gamma$ of the form $c/T$ for some nonnegative integer $c \le T$ (because the $\gammahat$ output by our algorithm has this form). But we have already proved this uniform convergence in Lemma \ref{lemma:s1}. In fact if we define the event:
\begin{align*}
B_1 (X) =  \Big\{ &\exists \, \pvecr \in \Delta_r ( \Hs )^m, \gamma \in  \{ 0, \frac{1}{T}, \frac{2}{T}, \ldots, 1 \}: \\
&\left\vert \underset{x \, \sim \Px}{\Ps} \left( \left\vert \Epsilon \left( x, \pvecr; \Qhat \right) - \gamma \right\vert >  4\alpha \right) - \underset{x \, \sim \Pxhat}{\Ps} \left( \left\vert \Epsilon \left( x, \pvecr; \Qhat \right) - \gamma \right\vert >  4\alpha \right) \right\vert > \frac{\beta}{2} \Big\}
\end{align*}
by Lemma \ref{lemma:s1}, so long as,
$$
n \ge \widetilde{O} \left( \frac{r \, m \, d_\Hs + \log \left(T/  \delta\right)}{ \beta^2} \right) \equiv \widetilde{O} \left( \frac{m \, d_\Hs + \log \left(1/ \nu^2 \delta\right)}{ \alpha^2 \beta^2} \right) \quad (\text{Recall: } T = O \left( \log \left( n \right) /\nu^2 \alpha^2 \right) )
$$
we have that $\Ps_X \left[ B_1 (X) \right] \le \delta/3$. And this implies
$$
\Ps_{X,\pvechatr} \left[ B \left( X, \pvechatr \right) \right] \le \Ps_{X} \left[ B_1 \left( X \right) \right] \le \delta/3
$$
\item term 2: Observe that
$$
\Ps_{X,\pvechatr} \left[ C \left( X, \pvechatr \right) \right] = \E_X \left[ \Ps_{\pvechatr} \left[ C \left( X, \pvechatr \right)  \right] \right]
$$
Notice when we condition on $X$, the conditional distribution of both $\pvechat$ and $\pvechatr$ will actually change because they both depend on $X$, but it is the case that for any $X$, $\pvechatr$ (given $X$) will be still independent draws from $\pvechat$ (given $X$) and this is what we actually need. We will show for any $X$:
$$
\Ps_{\pvechatr} \left[ C \left( X, \pvechatr \right)  \right] \le \delta/3
$$
Fix $X$. Let $X' = \{ x_i' \}_{i=1}^n$ be a \emph{new} data set of individuals drawn independently from the distribution $\Px$ and let $\Pxhat'$ denote the uniform distribution over $X'$. Define the events:
\begin{align*}
&C_1\left( X', \pvechatr \right) = \\
&\left\{ \left\vert \underset{x \sim \Px}{\Ps} \left( \left\vert \Epsilon \left( x,\pvechat ; \Qhat \right) - \Epsilon \left( x,\pvechat^{(r)}; \Qhat \right) \right\vert > \alpha \right) - \underset{x \sim \Pxhat'}{\Ps} \left( \left\vert \Epsilon \left( x,\pvechat ; \Qhat \right) - \Epsilon \left( x,\pvechat^{(r)}; \Qhat \right) \right\vert > \alpha \right) \right\vert > \frac{\beta}{2}\right\} \\
&C_2 \left( X', \pvechatr \right) =  \left\{ \underset{x \sim \Pxhat'}{\Ps} \left( \left\vert \Epsilon \left( x,\pvechat ; \Qhat \right) - \Epsilon \left( x,\pvechat^{(r)}; \Qhat \right) \right\vert > \alpha \right) \neq 0 \right\}
\end{align*}
It follows by a triangle inequality that:
\begin{align*}
\Ps_{\pvechatr} \left[ C \left( X, \pvechatr \right)  \right] &= \Ps_{X', \pvechatr} \left[ C \left( X, \pvechatr \right)  \right] \\
&\le \Ps_{X', \pvechatr} \left[ C_1\left( X', \pvechatr \right)  \right] + \Ps_{X', \pvechatr} \left[ C_2 \left( X', \pvechatr \right)  \right] \\
&= \E_{\pvechatr} \left[ \Ps_{X'} \left[ C_1\left( X', \pvechatr \right)  \right] \right] + \E_{X'} \left[ \Ps_{\pvechatr} \left[ C_2\left( X', \pvechatr \right)  \right] \right]
\end{align*}
Given $\pvechatr$, we have by a Chernoff-Hoeffding's inequality that as long as the sample complexity for $n$ is met:
$$
\Ps_{X'} \left[ C_1\left( X', \pvechatr \right)  \right] \le \delta/6
$$
On the other hand, given $X'$, it follows by Equation (\ref{eq:hey3}) that,
$$
\Ps_{\pvechatr} \left[ C_2\left( X', \pvechatr \right)  \right] \le \delta/6
$$
Hence:
$$
\Ps_{X,\pvechatr} \left[ C \left( X, \pvechatr \right) \right] \le \delta/3
$$
\item term 3: Equation (\ref{eq:hey3}) implies:
$$
\Ps_{X,\pvechatr} \left[ D \left( X, \pvechatr \right) \right] = \E_X \left[ \Ps_{\pvechatr} \left[ D \left( X, \pvechatr \right) \right] \right] \le \delta/3
$$
\end{itemize}
We finally proved $\Ps_X \left[ A \left( X \right) \right] \le \delta$. In other words, we have proved that for any $F$, with probability at least $1-\delta$ over the individuals $X$,
$$
\underset{x \, \sim \Px}{\Ps} \left( \left\vert \Epsilon \left( x, \pvechat; \Qhat \right) - \gammahat \right\vert > 5 \alpha \right) - \underset{x \, \sim \Pxhat}{\Ps} \left( \left\vert \Epsilon \left( x, \pvechat; \Qhat \right) - \gammahat \right\vert > 3 \alpha \right) \le \beta
$$
On the other hand, the in-sample guarantees provided in Theorem \ref{thm:insampleguarantees} implies for any $X$, with probability $1-\delta$ over the observed problems $F$, the pair $\left( \pvechat, \gammahat \right)$ of Algorithm \ref{algo} satisfies $$\underset{x \, \sim \Pxhat}{\Ps}  \left( \left\vert \Epsilon \left( x, \pvechat; \Qhat \right) - \gammahat \right\vert > 3 \alpha \right) = 0$$
Hence, as long as
\begin{equation}\label{eq:ncomplexity}
n \ge \widetilde{O} \left( \frac{m \, d_\Hs + \log \left(1/\nu^2 \delta\right)}{\alpha^2 \beta^2} \right)
\end{equation}
we have that with probability at least $1-2\delta$ over the observed data set $\left( X, F \right)$,
$$
\underset{x \, \sim \Px}{\Ps} \left( \left\vert \Epsilon \left( x, \pvechat; \Qhat \right) - \gammahat \right\vert > 5 \alpha \right) \le \beta
$$
which shows $\pvechat$, or equivalently the mapping $\psihat$, satisfies $(5\alpha, \beta)$-AIF with respect to the distributions $\Px$ and $\Qhat$. Now let's look at the generalization (with respect to $\Px$) error of $\pvechat$. Let
$$
\psi^\star = \psi^\star \left( \alpha; \Px, \Qhat \right),\quad \gamma^\star = \gamma^\star \left(\alpha; \Px, \Qhat \right), \quad \err \left( \psi^\star ; \Px, \Qhat \right) = \opt \left( \alpha ; \Px, \Qhat \right)
$$
be defined as in Definition \ref{def:opt}. Let also $\psi^\star \vert_F = \pvec^\star \in \Delta (\Hs)^m$. Since we are working with the empirical distribution of problems $\Qhat$,
$$
\err \left( \pvec^\star ; \Px, \Qhat \right) = \err \left( \psi^\star ; \Px, \Qhat \right) = \opt \left( \alpha ; \Px, \Qhat \right)
$$
We would like to compare the accuracy of $\pvechat$ with respect to the distributions $\Px$ and $\Qhat$, i.e. $\err \left( \pvechat; \Px, \Qhat\right)$, to $\opt \left( \alpha ; \Px, \Qhat \right)$ as a benchmark. We establish this comparison in three steps:
\begin{enumerate}
\item a bound on the difference between $\err \left( \pvechat; \Px, \Qhat\right)$ and its in-sample version $\err \left( \pvechat; \Pxhat, \Qhat\right)$. Note a uniform convergence can be achieved for all $\pvec \in \Delta (\Hs)^m$ by standard generalization techniques (Sauer's Lemma and the two-sample trick, as discussed in the proof of Lemma \ref{lemma:s1}), without even needing to consider the uniform convergence for $r$-sparse randomized classifiers. Because when investigating the concentration of the overall error rate which is in the form of nested expectations, the linearity of expectation helps us to directly turn a uniform convergence for all sets of pure classifiers $\hvec \in \Hs^m$ into a uniform convergence for all randomized classifiers $\pvec \in \Delta \left( \Hs \right)^m$ without blowing up the sample complexity by a factor of $1/\alpha^2$. We can therefore guarantee that for all $F$, with probability at least $1-\delta$ over $X$, as long as the sample complexity (\ref{eq:ncomplexity}) is satisfied:
\begin{equation}\label{eq:what4}
\left\vert \err \left( \pvechat ; \Px , \Qhat \right) - \err \left( \pvechat ; \Pxhat , \Qhat \right)\right\vert \le O \left( \alpha \beta \right)
\end{equation}
\item a bound on the difference between $\err \left( \pvec^\star; \Px, \Qhat\right)$ and its in-sample version $\err \left( \pvec^\star; \Pxhat, \Qhat\right)$. This is very simple and it follows immediately from the Chernoff-Hoeffding's theorem that for all $F$, with probability at least $1-\delta$ over $X$, as long as (\ref{eq:ncomplexity}) is satisfied:
$$
\left\vert \opt \left( \alpha ; \Px, \Qhat \right) - \err \left( \pvec^\star ; \Pxhat , \Qhat \right)\right\vert \le O \left( \alpha \beta \right)
$$
\item the pair $( \pvec^\star , \gamma^\star)$ is feasible in the empirical learning problem (\ref{box:fairerm}). As a consequence, Theorem \ref{thm:insampleguarantees} implies for any $X$, with probability $1-\delta$ over $F$:
$$
\err \left( \pvechat ; \Pxhat , \Qhat \right) \le \err \left( \pvec^\star ; \Pxhat , \Qhat \right) + 2 \nu
$$
\end{enumerate}
Finally, putting together all three pieces explained above, we have that with probability at least $1-3 \delta$ over the observed data set $(X, F)$,
$$
\err \left( \pvechat ; \Px , \Qhat \right) \le \opt \left( \alpha ; \Px, \Qhat \right) +  O \left( \nu \right) + O \left( \alpha \beta \right)
$$
completing the proof.
\end{proof}

\subsection{Proof of Theorem \ref{thm:generalizationf}: Generalization over $\Q$}
\begin{proof}[Proof of Theorem \ref{thm:generalizationf}]
Fix the set of observed individuals $X = \{x_i\}_{i=1}^n$ and recall $\Pxhat$ denotes the uniform distribution over $X$.  Let
$$
\psihat = \psihat \left( X, \What \right) \quad \text{where } \quad \What = \left\{ \wvec_t \right\}_{t=1}^T
$$ 
be the mapping returned by Algorithm \ref{algo}. We would like to lift the accuracy and fairness guarantees of $\psihat$ from the empirical distribution $\Qhat$ up to the true underlying distribution $\Q$. First observe that by the definition of $\psihat$ (Mapping \ref{algo:psihat}), for any $x \in \X$ and any distribution $Q$:
\begin{equation}\label{eq:what}
\Epsilon \left( x, \psihat ; \Q \right) = \frac{1}{T} \sum_{t=1}^T \Epsilon \left( x, \psihat_t ; \Q \right)
\end{equation}
 where $\psihat_{t} = \psihat \left( X, \wvec_t \right)$ is a mapping defined only by the weights $\wvec_t$ of round $t$ over the individuals $X$. $\psihat_t$ in fact takes a function $f$, solves one CSC problem on $X$ weighted by $\wvec_t$ and then returns the learned classifier. For any $t \in [T]$, we are interested in bounding the difference 
 $$
 \left\vert \Epsilon \left( x, \psihat_t ; \Q \right) - \Epsilon \left( x, \psihat_t ; \Qhat \right) \right\vert
 $$
Notice following the dynamics of our algorithm, $\wvec_t$ (and accordingly $\psihat_t$) has dependence on the batches of problems $\{ F_{t'} \}_{t'=1}^{t-1}$ and therefore we cannot directly invoke the Chernoff-Hoeffding's theorem to bound the above difference. We instead use the estimates $\Epsilon ( x, \psihat_t ; \Qhat_t )$ where $\Qhat_t = \uniform (F_t)$ as an intermediary step to bound the above difference. Observe that a simple triangle inequality implies
 $$
 \left\vert \Epsilon \left( x, \psihat_t ; \Q \right) - \Epsilon \left( x, \psihat_t ; \Qhat \right) \right\vert \le  \left\vert \Epsilon \left( x, \psihat_t ; \Q \right) - \Epsilon \left( x, \psihat_t ; \Qhat_t \right) \right\vert +  \left\vert \Epsilon \left( x, \psihat_t ; \Qhat_t \right) - \Epsilon \left( x, \psihat_t ; \Qhat \right) \right\vert
 $$
 Now we can invoke the Chernoff-Hoeffding's Theorem \ref{thm:chernoff} to bound each term appearing above separately. The batch of problems $F_t$ can be seen as independent draws from both distributions $\Q$ and $\Qhat$. It therefore follows that with probability $1-\delta$ over the draws of the batch $F_t$,
 $$
  \left\vert \Epsilon \left( x, \psihat_t ; \Q \right) - \Epsilon \left( x, \psihat_t ; \Qhat \right) \right\vert \le 2 \sqrt{ \frac{\log \left( 4/\delta \right)}{2 m_0} }
 $$
 We now use Equation (\ref{eq:what}) to translate this bound to a bound for $\psihat$. We have that for any $x \in \X$, with probability at least $1-\delta$ over the draws of all problems $F$,
 \begin{equation}\label{eq:what2}
 \left\vert \Epsilon \left( x, \psihat ; \Q \right) - \Epsilon \left( x, \psihat ; \Qhat \right) \right\vert \, \le \, 2 \sqrt{ \frac{\log \left( 4T/\delta \right)}{2 m_0} } 
 \end{equation}
And therefore, with probability at least $1-\delta$ over all problems $F$, for all $i \in [n]$,
\begin{equation}\label{eq:s3}
\left\vert \Epsilon \left( x_i, \psihat ; \Q \right) - \Epsilon \left( x_i, \psihat ; \Qhat \right) \right\vert \, \le \, 2 \sqrt{ \frac{\log \left( 4nT/\delta \right)}{2 m_0} } \, \le \, \nu \alpha
\end{equation}
where the second inequality follows from Assumption \ref{ass:m} on $m_0$. Now by a simple triangle inequality
\begin{align*}
&\underset{x \sim \Pxhat}{\Ps} \left( \left\vert \Epsilon \left( x, \psihat ; \Q \right) - \gammahat \right\vert > 4 \alpha \right) \\ &\le \underset{x \sim \Pxhat}{\Ps} \left( \left\vert \Epsilon \left( x, \psihat; \Q \right) - \Epsilon ( x, \psihat; \Qhat ) \right\vert > \alpha \right)  + \underset{x \sim \Pxhat}{\Ps} \left( \left\vert \Epsilon ( x, \psihat; \Qhat ) - \gammahat \right\vert > 3 \alpha \right)
\end{align*}
The first term appearing in the RHS of the above inequality is zero with probability $1-\delta$ over $F$ following Inequality (\ref{eq:s3}). The second term is zero with probability $1-\delta$ over $F$ as well following the in-sample guarantees provided in Theorem \ref{thm:insampleguarantees}.  Therefore, we conclude that for any set of observed individuals $X$, with probability at least $1-2\delta$ over the observed labelings $F$, as long as Assumption \ref{ass:m}:
\begin{equation}\label{eq:mcomplexity}
m = T \cdot m_0 \ge O \left( \frac{T \log \left( nT/\delta \right)}{\nu^2 \alpha^2} \right) = \Otilde \left( \frac{\log \left( n \right) \log \left( n/\delta \right)}{\nu^4 \alpha^4} \right)
\end{equation}
holds, the mapping $\psihat$ satisfies $\left(4\alpha,0\right)$-AIF with respect to the distributions $\left( \Pxhat, \Q \right)$:
$$
\underset{x \sim \Pxhat}{\Ps} \left( \left\vert \Epsilon \left( x, \psihat; \Q \right) - \gammahat \right\vert > 4 \alpha \right) = 0
$$
Now let's look at the generalization (with respect to $\Q$) error of  $\psihat$. Let 
$$\psi^\star = \psi^\star \left( \alpha; \Pxhat, \Q \right), \quad \gamma^\star = \gamma^\star \left( \alpha; \Pxhat, \Q \right), \quad \err \left( \psi^\star ; \Pxhat, \Q \right) = \opt \left( \alpha ; \Pxhat, \Q \right)
$$
be defined as in Definition \ref{def:opt}. We would like to compare the accuracy of $\psihat$ with respect to the distributions $\Pxhat$ and $\Q$, i.e. $\err \left( \psihat; \Pxhat, \Q \right)$, to $\opt \left( \alpha ; \Pxhat, \Q \right)$ as a benchmark. We follow the same approach we used in the proof of Theorem \ref{thm:generalizationx} to achieve an upper bound for the difference between $\psihat$'s error and the optimal error. It follows directly from Inequality (\ref{eq:s3}) that with probability $1-\delta$ over the problems $F$, as long as the sample complexity (\ref{eq:mcomplexity}) is satisfied,
\begin{align}\label{eq:s4}
\left\vert \err \left( \psihat ; \Pxhat , \Q \right) - \err \left( \psihat ; \Pxhat , \Qhat \right)\right\vert &=  \left\vert \frac{1}{n} \sum_{i=1}^n \Epsilon \left( x_i, \psihat ; \Q \right) - \frac{1}{n} \sum_{i=1}^n  \Epsilon \left( x_i, \psihat ; \Qhat \right)\right\vert \, \le \, \nu \alpha
\end{align}
and that by a simple application of the Chernoff-Hoeffding's inequality, as long as the sample complexity (\ref{eq:mcomplexity}) is satisfied, with probability $1-\delta$ over $F$ we have that
\begin{equation}\label{eq:s5}
\quad \left\vert \err \left( \psi^\star ; \Pxhat , \Q \right) - \err \left( \psi^\star ; \Pxhat , \Qhat \right)\right\vert \le O \left(\nu \alpha \right)
\end{equation}
On the other hand, with probability $1-\delta$ over $F$, the pair $\left( \psi^\star \vert_F, \gamma^\star \right)$ -- where $\psi^\star \vert_F$ is the mapping $\psi^\star$ restricted to the problems $F$ --  is feasible in the empirical learning problem (\ref{box:fairerm}). Because, with probability $1-\delta$ as long as the sample complexity (\ref{eq:mcomplexity}) is satisfied:
\begin{align*}
\left\vert \Epsilon \left(x_i, \psi^\star \vert_F ; \Qhat \right) - \gamma^\star \right\vert &\le \left\vert \Epsilon \left(x_i, \psi^\star \vert_F ; \Qhat \right) - \Epsilon \left(x_i, \psi^\star ; \Q \right) \right\vert + \left\vert \Epsilon \left(x_i, \psi^\star ; \Q \right) - \gamma^\star \right\vert  \le \alpha + \alpha = 2 \alpha
\end{align*}
Therefore, by Theorem \ref{thm:insampleguarantees}, we have that with probability $1-\delta$ over $F$,
\begin{equation}\label{eq:s6}
\err \left( \psihat ; \Pxhat , \Qhat \right) = \err \left( \pvechat ; \Pxhat , \Qhat \right) \le \err \left( \psi^\star \vert_F ; \Pxhat , \Qhat \right) + 2 \nu = \err \left( \psi^\star  ; \Pxhat , \Qhat \right) + 2 \nu
\end{equation}
Putting together Inequalities (\ref{eq:s4}), (\ref{eq:s5}), and (\ref{eq:s6}), we conclude that for any $X$, with probability at least $1-4\delta$ over the problems $F$,
$$
\err \left( \psihat ; \Pxhat , \Q \right) \le \opt \left( \alpha ; \Pxhat, \Q \right) + O \left( \nu \right)
$$
\end{proof}

\subsection{Proof of Theorem \ref{thm:generalizationxf}: Simultaneous Generalization over $\Px$ and $\Q$}
We have proved so far:
\begin{itemize}
\item Theorem \ref{thm:generalizationx} (Generalization over $\Px$): $\left( \Pxhat , \Qhat \right) \overset{\text{lifted}}{\longrightarrow} \left( \Px, \Qhat \right)$.
\item Theorem \ref{thm:generalizationf} (Generalization over $\Q$): $\left( \Pxhat, \Qhat \right)  \overset{\text{lifted}}{\longrightarrow} \left( \Pxhat, \Q \right)$
\end{itemize}
Before we prove Theorem \ref{thm:generalizationxf}  which considers simultaneous generalization over both distributions $\Px$ and $\Q$, we first prove the following Lemma \ref{lemma:concentration2}. Notice when proving the final generalization theorem, we take the following natural two steps in lifting the guarantees:
$$
\left( \Pxhat , \Qhat \right) \rightarrow \left( \Px, \Qhat \right) \rightarrow \left( \Px, \Q \right)
$$
The first step of the above scheme is exactly what we have proved in Theorem \ref{thm:generalizationx}. However, we cannot directly invoke Theorem \ref{thm:generalizationf} to prove the second step since we had the empirical distribution of individuals (i.e. $\Pxhat$) in that theorem. In the following Lemma, we basically prove the second step where the distribution over individuals is $\Px$.
\medskip

\begin{lemma}\label{lemma:concentration2}
Suppose
$$
n \ge \widetilde{O} \left( \frac{m \, d_\Hs + \log \left(1/\nu^2 \delta\right)}{\alpha^2 \beta^2} \right) \quad , \quad m \ge \Otilde \left( \frac{\log \left( n \right) \log \left( n/\delta \right)}{\nu^4 \alpha^4} \right)
$$
and let $\psihat = \psihat \left( X, \What \right)$ be the output of Algorithm \ref{algo}. We have that for any $X$, with probability at least $1-2\delta$ over the problems $F$,
$$
 \underset{x \, \sim \Px}{\Ps} \left( \left\vert \Epsilon \left( x, \psihat; \Q \right) - \Epsilon \left( x, \psihat; \Qhat \right)  \right\vert > \alpha  \right) \le \beta
$$
and that
$$
\left\vert \err \left( \psihat; \Px, \Q \right) - \err \left( \psihat ; \Px, \Qhat \right) \right\vert \le O \left( \alpha \beta \right) + O \left( \nu \alpha  \right)
$$
\end{lemma}

\begin{proof}
Let's prove the first part of the Lemma. Let $X' = \{ x_i'\}_{i=1}^n \subseteq \X^n$ be -- any -- set of $n$ individuals. Recall that Inequality (\ref{eq:what2}) developed in the proof of Theorem \ref{thm:generalizationf} implies with probability $1-\delta$ over the problems $F$, for all $i \in [n]$:
$$
 \left\vert \Epsilon \left( x_i', \psihat ; \Q \right) - \Epsilon \left( x_i', \psihat ; \Qhat \right) \right\vert \, \le \, 2 \sqrt{ \frac{\log \left( 4nT/\delta \right)}{2 m_0} } \, \le \, \nu \alpha
$$
In other words, if $\Pxhat'$ represents the uniform distribution over $X' $, we have that for any $X' \subseteq \X^n$, with probability at least $1-\delta$ over $F$,
\begin{equation}\label{eq:what3}
\underset{x \sim \Pxhat'}{\Ps} \left(  \left\vert \Epsilon \left( x, \psihat ; \Q \right) - \Epsilon \left( x, \psihat ; \Qhat \right) \right\vert  > \alpha \right) = 0
\end{equation}
This doesn't give us what we want because this inequality works for any distribution with support of size at most $n$ while we want a guarantee for any $x \sim \Px$ and that $\Px$ might have an infinite-sized support. For random variables $F$, consider the event $A \left( F \right)$:
$$
A \left( F \right) = \left\{  \underset{x \, \sim \Px}{\Ps} \left( \left\vert \Epsilon \left( x, \psihat; \Q \right) - \Epsilon \left( x, \psihat; \Qhat \right)  \right\vert > \alpha  \right) > \beta \right\}
$$
We eventually want to show that $\Ps_F \left[ A \left( F\right) \right]$ is small. Following the guarantee of (\ref{eq:what3}) we consider an auxiliary data set of individuals $X'$ to argue that $\Ps_F \left[ A \left( F\right) \right]$ is in fact small. To formalize our argument, let $X' = \{ x_i' \}_{i=1}^n \sim \Px^n$ be a \emph{new} set of $n$ individuals drawn independently from the distribution $\Px$ and let $\Pxhat'$ denote the uniform distribution over $X'$. Define events $B \left( F, X' \right)$ and $C \left(F,X'\right)$ which depend on both $F$ and $X'$:
\begin{align*}
& B \left( F, X' \right) = \\
& \ \ \ \left\{ \left\vert \underset{x \, \sim \Px}{\Ps} \left( \left\vert \Epsilon \left( x, \psihat; \Q \right) - \Epsilon \left( x, \psihat; \Qhat \right)  \right\vert > \alpha  \right) - \underset{x \, \sim \Pxhat'}{\Ps} \left( \left\vert \Epsilon \left( x, \psihat; \Q \right) - \Epsilon \left( x, \psihat; \Qhat \right)  \right\vert > \alpha  \right) \right\vert > \beta \right\} \\
&C \left(F,X'\right) = \left\{ \underset{x \, \sim \Pxhat'}{\Ps} \left( \left\vert \Epsilon \left( x, \psihat; \Q \right) - \Epsilon \left( x, \psihat; \Qhat \right)  \right\vert > \alpha  \right) \neq 0 \right\}
\end{align*}
We have that
\begin{align*}
\Ps_F \left[ A \left( F\right) \right] &= \Ps_{F,X'} \left[ A \left( F\right) \right] \\
&\le \Ps_{F,X'} \left[ B \left( F, X' \right) \right] + \Ps_{F,X'} \left[ C \left(F,X'\right) \right]  \quad (\text{follows from a triangle inequality})\\
&= \E_F \left[ \Ps_{X'} \left[ B \left( F, X' \right) \right] \right] + \E_{X'} \left[ \Ps_{F} \left[ C \left(F,X'\right) \right] \right]
\end{align*}
But Inequality (\ref{eq:what3}) implies for any $X'$, $\Ps_{F} \left[ C \left(F,X'\right) \right] \le \delta$. And that by a simple application of the Chernoff-Hoeffding's inequality, as long as the sample complexity for $n$ stated in the Lemma is met, we have that  for any $F$, $\Ps_{X'} \left[ B \left( F, X' \right) \right] \le \delta$. We have therefore shown that
$$
\Ps_F \left[ A \left( F\right) \right] \le  2\delta
$$
proving the first part of the Lemma. The second part of the Lemma can be proved similarly with the same idea of considering an auxiliary data set $X' \sim \Px^n$. Skipping the details, we can in fact show that with probability $1-2\delta$ over $F$,
\begin{align*}
&\left\vert \err \left( \psihat; \Px, \Q \right) - \err \left( \psihat ; \Px, \Qhat \right) \right\vert \\
&\le \underset{x \sim \Px}{\E} \left\vert \Epsilon \left( x, \psihat; \Q \right) - \Epsilon \left( x, \psihat ;\Qhat \right) \right\vert \\
&\le \left\vert  \underset{x \sim \Px}{\E} \left\vert \Epsilon \left( x, \psihat; \Q \right) - \Epsilon \left( x, \psihat ;\Qhat \right) \right\vert -  \underset{x \sim \Pxhat'}{\E} \left\vert \Epsilon \left( x, \psihat; \Q \right) - \Epsilon \left( x, \psihat ;\Qhat \right) \right\vert \right\vert\\
&+  \underset{x \sim \Pxhat'}{\E} \left\vert \Epsilon \left( x, \psihat; \Q \right) - \Epsilon \left( x, \psihat ;\Qhat \right) \right\vert \\
&\le O \left( \alpha \beta \right) + O \left( \nu \alpha  \right)
\end{align*}
\end{proof}

\begin{proof}[Proof of Theorem \ref{thm:generalizationxf}]
First, observe that by a traingle inequality
\begin{align*}
&  \underset{x \, \sim \Px}{\Ps} \left( \left\vert \Epsilon \left( x, \psihat; \Q \right) - \gammahat \right\vert > 6 \alpha  \right)  \\
& \le \underset{x \, \sim \Px}{\Ps} \left( \left\vert \Epsilon \left( x, \psihat; \Q \right) - \Epsilon \left( x, \psihat; \Qhat \right)  \right\vert > \alpha  \right) + \underset{x \, \sim \Px}{\Ps} \left( \left\vert \Epsilon \left( x, \psihat; \Qhat \right) - \gammahat \right\vert > 5\alpha  \right)
\end{align*}
We will argue that as long as,
$$
n \ge \widetilde{O} \left( \frac{m \, d_\Hs + \log \left(1/\nu^2 \delta\right)}{\alpha^2 \beta^2} \right) \quad , \quad m \ge \Otilde \left( \frac{ \log \left( n \right) \log \left( n/\delta \right)}{\nu^4 \alpha^4} \right)
$$
we can further bound the above inequality by $2\beta$. Lemma \ref{lemma:concentration2} implies for any $X$, the first term appearing in the RHS of the above inequality is bounded by $\beta$ with probability $1-2\delta$ over the problems $F$. The second term is bounded by $\beta$ as well by Theorem \ref{thm:generalizationx} with probability $1-3\delta$ over the observed individuals and problems $(X,F)$ (in the theorem we actually state w.p. $1-6\delta$ because $3\delta$ comes from the accuracy guarantee). Putting these two pieces together, we have that with probability at least $1-5\delta$ over the data set $(X,F)$,
$$
\underset{x \, \sim \Px}{\Ps} \left( \left\vert \Epsilon \left( x, \psihat; \Q \right) - \gammahat \right\vert > 6 \alpha  \right) \le 2 \beta
$$
which implies $\psihat$ satisfies $\left(6\alpha, 2\beta \right)$-AIF with respect to the distributions $\left( \Px , \Q \right)$. Now let's look at the generalization error of $\psihat$. Let 
$$
\psi^\star = \psi^\star \left( \alpha; \Px, \Q \right), \quad \gamma^\star = \gamma^\star \left( \alpha; \Px, \Q \right), \quad \err \left( \psi^\star ; \Px, \Q \right) = \opt \left( \alpha ; \Px, \Q \right)
$$
be defined as in Definition \ref{def:opt}. We will again use the triangle inequality to write:
\begin{align*}
\left\vert \err \left( \psihat; \Px, \Q \right) - \err \left( \psihat ; \Pxhat, \Qhat \right) \right\vert &\le \left\vert \err \left( \psihat; \Px, \Q \right) - \err \left( \psihat ; \Px, \Qhat \right) \right\vert \\
&+ \left\vert \err \left( \psihat; \Px, \Qhat \right) - \err \left( \psihat ; \Pxhat, \Qhat \right) \right\vert
\end{align*}
It follows from Lemma \ref{lemma:concentration2} that the first term appearing in the RHS is bounded by $O \left( \alpha \beta \right) + O \left( \nu \alpha  \right)$ with probability $1-2\delta$ over the problems $F$, and we also have that the second term is bounded by $O \left( \alpha \beta \right)$ as well with probability $1-\delta$ over the individuals $X$ (see (\ref{eq:what4})). Therefore, with probability $1-3\delta$ over the data set $(X,F)$,
$$
\left\vert \err \left( \psihat; \Px, \Q \right) - \err \left( \psihat ; \Pxhat, \Qhat \right) \right\vert \, \le \, O \left( \alpha \beta \right) + O \left( \nu \alpha  \right)
$$
It follows similarly that with probability $1-2\delta$ over $X$ and $F$,
$$
\left\vert \err \left( \psi^\star ; \Px, \Q \right) - \err \left( \psi^\star  ; \Pxhat, \Qhat \right) \right\vert \, \le \, O \left( \alpha \beta \right) + O \left( \nu \alpha  \right)
$$
and as in the proof of Theorem \ref{thm:generalizationf}, we can show that with probability $1-\delta$ over $F$, the pair $\left( \psi^\star \vert_F, \gamma^\star \right)$ -- where $\psi^\star \vert_F$ is the mapping $\psi^\star$ restricted to the domain $F$ --  is feasible in the empirical fair learning problem (\ref{box:fairerm}), which again implies by Theorem \ref{thm:insampleguarantees} that with probability $1-\delta$ over $F$
$$
\err \left( \psihat ; \Pxhat , \Qhat \right) = \err \left( \pvechat ; \Pxhat , \Qhat \right) \le \err \left( \psi^\star \vert_F ; \Pxhat , \Qhat \right) + 2 \nu = \err \left( \psi^\star  ; \Pxhat , \Qhat \right) + 2 \nu
$$
Putting all inequalities together, we have that with probability $1-7\delta$ over $(X,F)$,
$$
\err \left( \psihat ; \Px , \Q \right) \, \le \, \opt \left( \alpha ; \Px, \Q \right) + O \left( \nu \right) + O \left( \alpha \beta \right)
$$
\end{proof}

\section{Proofs of Section \ref{sec:learningfpaif}}\label{app:learningfpaif}
We prove only the regret bounds of the players (Lemma \ref{lemma:regretlearnerfp} and \ref{lemma:regretdualfp}) in this section. The rest of the proofs are similar to their counterparts in the AIF section.
\subsection{Regret Bounds}

\begin{proof}[Proof of Lemma \ref{lemma:regretlearnerfp}]
Recall the Learner is accumulating regret only because they are using the given quantities $\{ \rhotilde_i \}_{i=1}^n$ instead of $ \{ \rhohat_i \}_{i=1}^n$ that depend on the observed data $F$. In other words, at round $t \in [T]$ of the algorithm  instead of best responding to
\begin{align*}
& \Ls \left( \hvec, \gamma , \lamb_t \right) = -\gamma \sum_{i=1}^n w_{i,t} \\
&  + \frac{1}{m}\sum_{j=1}^m  \left\{ \sum_{i=1}^n \left(\frac{1}{n} +  \frac{w_{i,t}}{\rhohat_{i}}\right) \1 \left[ h_j (x_i) = 1, f_j(x_i) = 0 \right] + \left( \frac{1}{n} \right) \1 \left[ h_j (x_i) = 0, f_j(x_i) = 1 \right] \right\}
\end{align*}
The Learner is in fact minimizing
\begin{align*}
& \widetilde{\Ls} \left( \hvec, \gamma , \lamb_t \right) = -\gamma \sum_{i=1}^n w_{i,t} \\
& + \frac{1}{m}\sum_{j=1}^m  \left\{ \sum_{i=1}^n \left(\frac{1}{n} +  \frac{w_{i,t}}{\rhotilde_{i}}\right) \1 \left[ h_j (x_i) = 1, f_j(x_i) = 0 \right] + \left( \frac{1}{n} \right) \1 \left[ h_j (x_i) = 0, f_j(x_i) = 1 \right] \right\}
\end{align*}
Observe that for any pair of $\left( \hvec, \gamma \right)$,
\begin{align*}
\left\vert \Ls \left( \hvec, \gamma , \lamb_t \right) - \widetilde{\Ls}  \left( \hvec, \gamma , \lamb_t \right) \right\vert &\le \frac{1}{m}\sum_{j=1}^m   \sum_{i=1}^n \left\vert \frac{w_{i,t}}{\rhotilde_{i}} -  \frac{w_{i,t}}{\rhohat_{i}} \right\vert \1 \left[ h_j (x_i) = 1, f_j(x_i) = 0 \right] \\
& \le \sum_{i=1}^n \left\vert \frac{w_{i,t}}{\rhotilde_{i}} -  \frac{w_{i,t}}{\rhohat_{i}} \right\vert \frac{1}{m}\sum_{j=1}^m \1 \left[f_j(x_i) = 0 \right] \\
&= \sum_{i=1}^n \rhohat_i \left\vert \frac{w_{i,t}}{\rhotilde_{i}} -  \frac{w_{i,t}}{\rhohat_{i}} \right\vert \\
& \le \sum_{i=1}^n \frac{\left\vert \rhohat_i - \rhotilde_{i} \right\vert \cdot \left\vert w_{i,t} \right\vert  }{\rhotilde_{i}}
\end{align*}
A simple application of Chernoff-Hoeffding's inequality (Theorem \ref{thm:chernoff}) implies with probability $1-\delta$ over the observed labelings $F$, uniformly for all $i$,
$$
\left\vert \rhohat_i - \rho_i \right\vert \le \sqrt{\frac{\log\left( 2n / \delta \right)}{2 m_0}}
$$
On the other hand by Assumption \ref{ass:rho}, we have that for all $i \in [n]$,
$$
\left\vert \rhotilde_i - \rho_i \right\vert \le \sqrt{\frac{\log\left( n \right)}{2m_0}}
$$
Therefore for a given $t$, with probability $1-\delta$ over $F$,
\begin{align*}
\left\vert \Ls \left( \hvec, \gamma , \lamb_t \right) - \widetilde{\Ls}  \left( \hvec, \gamma , \lamb_t \right) \right\vert &\le \frac{2}{\rhotildemin}\sqrt{\frac{\log\left( 2n / \delta \right)}{2 m_0}} \sum_{i=1}^n \left\vert w_{i,t} \right\vert \\
&= \frac{2}{\rhotildemin}\sqrt{\frac{\log\left( 2n / \delta \right)}{2 m_0}} \sum_{i=1}^n \left\vert \lambda_{i,t}^+ - \lambda_{i,t}^- \right\vert\\
&\le \frac{2 B}{\rhotildemin}\sqrt{\frac{\log\left( 2n / \delta \right)}{2 m_0}}
\end{align*}
Now, let $\left( \hvec_t^\star, \gamma_t^\star \right) = \argmin_{\hvec, \gamma} \Ls \left( \hvec, \gamma , \lamb_t \right)$ be the best response of the Learner to $\Ls$ and recall the actual play of the Learner is in fact $\left( \hvec_t, \gamma_t \right) = \argmin_{\hvec, \gamma} \widetilde{\Ls} \left( \hvec, \gamma , \lamb_t \right)$. Observe that the above inequality implies for any $t$, with probability $1-\delta$ over $F$,
\begin{align*}
\Ls \left( \hvec_t, \gamma_t, \lamb_t \right) &\le \widetilde{\Ls} \left( \hvec_t, \gamma_t, \lamb_t \right) + \frac{2 B}{\rhotildemin}\sqrt{\frac{\log\left( 4n / \delta \right)}{2 m_0}} \\
& \le \widetilde{\Ls} \left( \hvec_t^\star , \gamma_t^\star , \lamb_t \right) + \frac{2 B}{\rhotildemin}\sqrt{\frac{\log\left( 4n / \delta \right)}{2 m_0}} \\
& \le {\Ls} \left( \hvec_t^\star , \gamma_t^\star , \lamb_t \right) + \frac{4 B}{\rhotildemin}\sqrt{\frac{\log\left( 4n / \delta \right)}{2 m_0}}
\end{align*}
It follows immediately that with probability $1-\delta$ over the observed labelings $F$, for all $ t \in [T]$
$$
\Ls \left( \hvec_t, \gamma_t, \lamb_t \right) - \min_{\pvec, \gamma} \Ls \left( \pvec , \gamma, \lamb_t \right) \le \frac{4 B}{\rhotildemin}\sqrt{\frac{\log\left( 4nT/ \delta \right)}{2 m_0}}
$$
which completes the proof.
\end{proof}

\begin{proof}[Proof of Lemma \ref{lemma:regretdualfp}]
The proof is almost similar to that of Lemma \ref{lemma:regretdual}. The terms
$
\frac{B \log \left(2n +1 \right)}{\eta T} + \eta \left( 1 + 2 \alpha \right)^2 B
$
come from the regret analysis of the exponentiated gradient descent algorithm.The only difference is that now we have to bound the difference between false positive rates, one with respect to the distribution $\Qhat$ and another with respect to the distribution $\Qhat_t$, for all $i$ and $t$:
$$
\left\vert \Epsilonfp \left(x_i, \psihat_t ; \Qhat \right) - \Epsilonfp \left(x_i, \psihat_t ; \Qhat_t \right) \right\vert
$$
Let $\rhohat_{i,t} = {\Ps_{f \sim \Qhat_t}} \left[ f(x_i) = 0 \right]$ and recall $\rhohat_{i} = {\Ps_{f \sim \Qhat}} \left[ f(x_i) = 0 \right]$. We have that
$$
\Epsilonfp \left(x_i, \psihat_t ; \Qhat \right) = (\rhohat_{i})^{-1} \cdot \underset{f \sim \Qhat}{\E} \left[ \underset{h \sim \psihat_{t,f}}{\Ps} \left[ h(x_i) = 1, \, f(x_i) = 0 \right] \right]
$$
$$
\Epsilonfp \left(x_i, \psihat_t ; \Qhat_t \right) = (\rhohat_{i,t})^{-1} \cdot \underset{f \sim \Qhat_t}{\E} \left[ \underset{h \sim \psihat_{t,f}}{\Ps} \left[ h(x_i) = 1, \, f(x_i) = 0 \right] \right]
$$
And therefore
\begin{align*}
& \left\vert \Epsilonfp \left(x_i, \psihat_t ; \Qhat \right) - \Epsilonfp \left(x_i, \psihat_t ; \Qhat_t \right) \right\vert \\
& \le \frac{1}{\rhohat_i} \left\{ \left\vert \underset{f \sim \Qhat}{\E} \left[ \underset{h \sim \psihat_{t,f}}{\Ps} \left[ h(x_i) = 1, \, f(x_i) = 0 \right] \right] - \underset{f \sim \Qhat_t}{\E} \left[ \underset{h \sim \psihat_{t,f}}{\Ps} \left[ h(x_i) = 1, \, f(x_i) = 0 \right] \right] \right\vert + \left\vert \rhohat_{i} - \rhohat_{i,t}\right\vert \right\}
\end{align*}
This inequality follows from the simple fact that for $a \le b$ and $c \le d$, all in $ \R_+$:
$$
\left\vert \frac{a}{b} - \frac{c}{d} \right\vert = \left\vert \frac{a}{b} - \frac{c}{b} + \frac{c}{b} - \frac{c}{d} \right\vert \le \frac{1}{b} \left\vert a - c \right\vert + c \left\vert \frac{1}{b} - \frac{1}{d} \right\vert \le \frac{1}{b} \left\vert a - c \right\vert  + \frac{1}{b} \left\vert b - d \right\vert
$$
Now by viewing the batch $F_t$ as a random sample from the set of problems $F$, we can apply the Chernoff-Hoeffding's inequality to bound each difference appearing above separately. We therefore have that with probability $1-\delta$ over the draws of $F_t$:
$$
\left\vert \Epsilonfp \left(x_i, \psihat_t ; \Qhat \right) - \Epsilonfp \left(x_i, \psihat_t ; \Qhat_t \right) \right\vert \le \frac{2}{\rhohat_i} \sqrt{\frac{\log \left(4/\delta \right)}{2 m_0}}
$$
The rest is similar to how the proof for Lemma \ref{lemma:regretdual} proceeds.
\end{proof}

\section{Additional Experimental Results}

\begin{figure*}[h]
\centering
	\begin{subfigure}[t]{0.45\textwidth}
	\centering
	\includegraphics[scale=0.39]{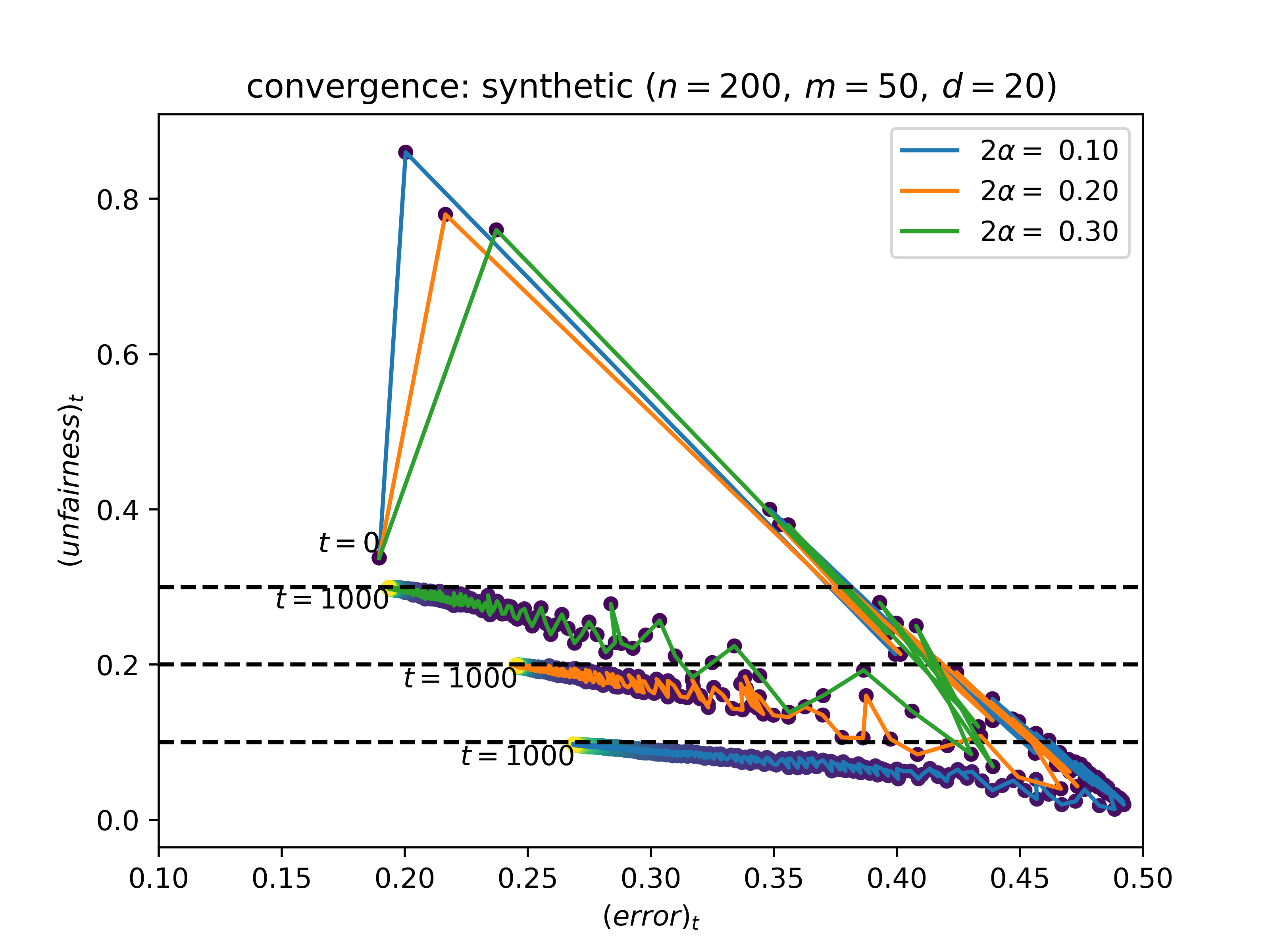}
	\caption{convergence plot: synthetic data}
	\end{subfigure}
	\hspace{1cm}
	\begin{subfigure}[t]{0.45\textwidth}
	\centering
	\includegraphics[scale=0.39]{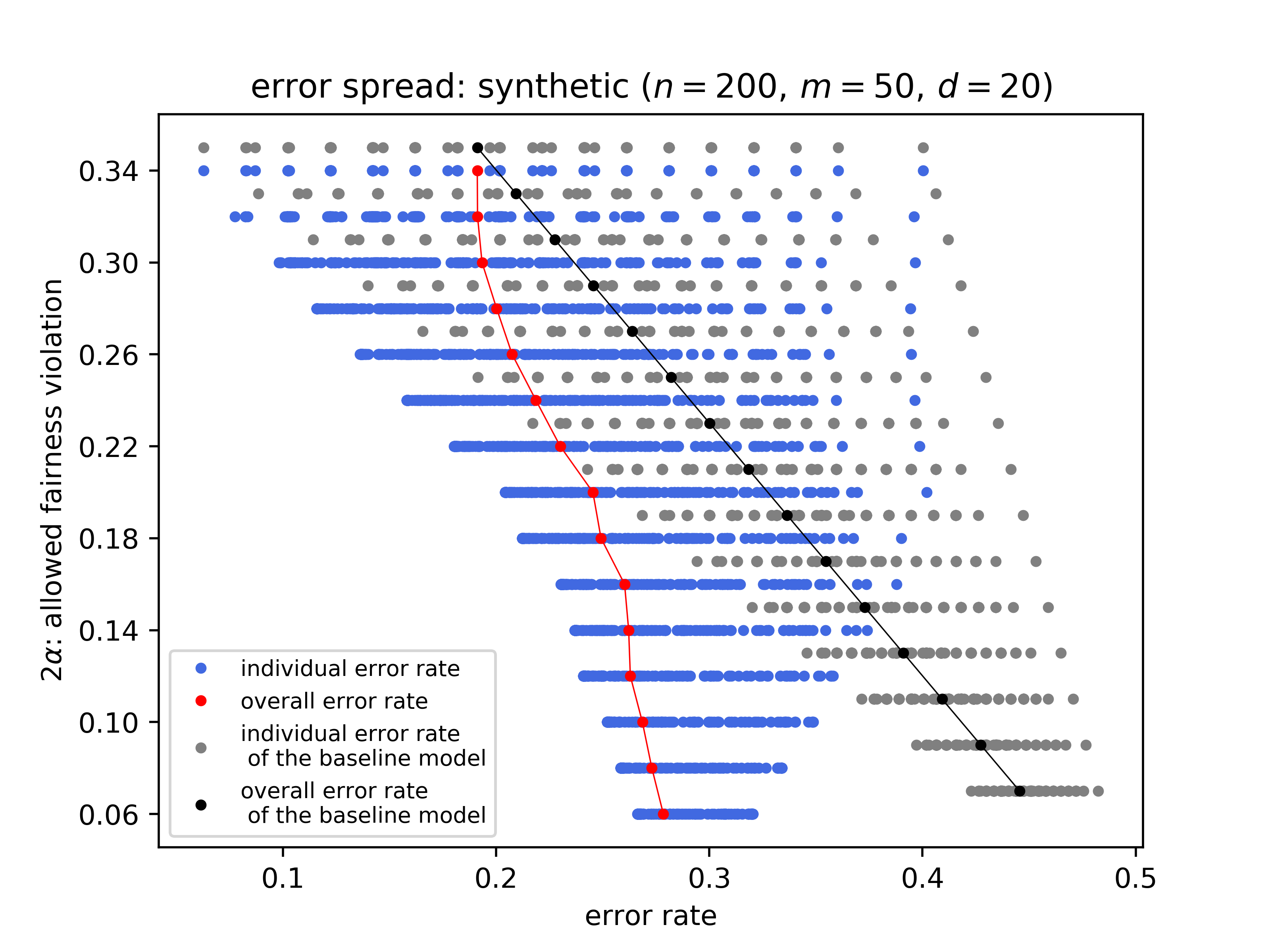}
	\caption{error spread plot: synthetic data}
	\end{subfigure}
\caption{
(a) Error-unfairness trajectory plots illustrating the convergence of algorithm \textbf{AIF-Learn}.
(b) Error-unfairness tradeoffs and individual errors for \textbf{AIF-Learn} vs. simple mixtures
	of the error-optimal model and random classification. Gray dots are shifted upwards slightly
	to avoid occlusions.
}
\label{fig:synthetic}
\end{figure*}


In addition to the experiments on the Communities and Crime data set, we ran additional experiments on synthetic data sets to further verify the performance of our algorithms empirically. Given the parameters $n$ (number of individuals), $m$ (number of problems), $d$ (dimension), and a coin with bias $q \ge 0.5$, an instance of the synthetic data is generated as follows: 

First $n$ individuals $\{x_i\}_{i=1}^n$ are generated randomly where $x_i \in \{ \pm 1 \}^d$. Let the first $75 \%$ of the individuals be called ``group 1'' and the rest ``group 2''. For each learning task $j \in [m]$, we randomly sample two weight vectors $w_{j,\text{maj}}, w_{j,\text{min}} \in \{ \pm 1 \}^d$. Now consider a matrix $Y \in \{0,1\}^{n \times m}$ that needs to be filled in with labels. We intentionally inject unfairness into the data set (because otherwise fairness would be achieved for free due to the random generation of data) by considering ``majority'' and ``minority'' groups for each learning task (majority group will be labelled using $w_{j,\text{maj}}$ and the minority by $w_{j,\text{min}}$). For each learning task $j$, individuals in group 1 get $q$ chance of falling into the majority group of that task and individuals in group 2 get $1-q$ chance (this is implemented by flipping a coin with bias $q$ for each individual). Once the majority and minority groups for each learning task are specified, the labels $Y = [y_{ij}]$ are set according to $y_{ij} = \text{sign} (w_{i,j}^T \, x_i)$ where $w_{i,j} \in \{ w_{j,\text{maj}}, w_{j,\text{min}}\}$. (Notice $q=1$ corresponds to a fixed (across learning tasks) majority group of size $0.75n$, and $q = 0.5$ corresponds to the case where the majority groups are formed completely at random and that they have size $0.5n$ in expectation.)

It is ``expected'' that solving the unconstrained learning problem on an instance of the above synthetic data generation process would result in discrimination against the minority groups which are formed mostly by the individuals in group 2, and therefore our algorithm should provide a fix for this unfair treatment. We therefore took an instance of the synthetic data set with parameters $n=200, m = 50, d = 20, q = 0.8$ and ran our algorithm on varying values of allowed fairness violation $2\alpha$ (notice on an input $\alpha$ to the algorithm, the individual error rates are allowed to be within at most $2\alpha$ of each other). We present the results (which are reported in-sample) in Figure \ref{fig:synthetic} where similar to the results on the Communities and Crime data set, we get convergence after $T=1000$ iterations (or 50,000 calls to the learning oracle, see Figure \ref{fig:synthetic}(a)). Furthermore, Figure \ref{fig:synthetic}(b) suggests that our algorithm on the synthetic data set considerably outperforms the baseline that only mixes the error-optimal model with random classification.

%% file: FPAIF.tex

\begin{definition}[Individual False Positive/Negative Error Rates]\label{def:individualerrorfp}
For a given individual $x \in \X$, a mapping $\psi \in \Delta (\Hs)^\F$, and distribution $\Q$ over the space of problems, the Individual false positive/negative rate incurred by $\psi$ on $x$ are defined as follows:
$$
\Epsilonfp \left(x, \psi; \Q \right) = \frac{1}{\underset{f \sim \Q}{\Ps} \left[ f(x) = 0\right]}  \cdot \underset{f \sim \Q}{\E} \left[ \underset{h \sim \psi_f}{\Ps} \left[ h(x) = 1, \, f(x) = 0 \right] \right]
$$
$$
\Epsilonfn \left(x, \psi ; \Q \right) = \frac{1}{\underset{f \sim \Q}{\Ps} \left[ f(x) = 1\right]} \cdot \underset{f \sim \Q}{\E} \left[ \underset{h \sim \psi_f}{\Ps} \left[ h(x) =0, \, f(x) = 1 \right] \right]
$$
\end{definition}

\begin{definition}[FPAIF fairness notion]\label{def:fairnessfp}
In our framework, we say a mapping $\psi \in \Delta (\Hs)^\F$ satisfies ``$(\alpha, \beta)$-FPAIF'' (reads $(\alpha,\beta)$-approximate False Positive Average Individual Fairness) with respect to the distributions $\left( \Px, \Q \right)$ if there exists $\gamma \ge 0$ such that
\begin{equation*}
\underset{x \sim \Px}{\Ps} \left( \left\vert \Epsilonfp \left(x,\psi ; \Q \right) - \gamma \right\vert > \alpha \right) \le \beta
\end{equation*}
\end{definition}

In this section we consider learning subject to the FPAIF notion of fairness. The FPAIF fairness notion basically asks that the individual false positive rates be approximately (corresponds to $\alpha$) equalized across almost all (corresponds to $\beta$) individuals. Learning subject to equalizing false negative rates can be developed similarly. We will be less wordy in this section as the ideas and the approach that we take are mostly similar to those developed in Section \ref{sec:learningaif}. We start off by casting the fair learning problem as the constrained optimization problem (\ref{box:fairfp}) where a mapping $\psi$ is to be found such that all individual false positive rates incurred by $\psi$ are within $\alpha$ of some $\gamma$. As before, we denote the optimal error of the optimization problem (\ref{box:fairfp}) by $\opt$ and will consider that as a benchmark to evaluate the  accuracy of our algorithm's trained mapping.

\begin{tcolorbox}[title= {Fair Learning Problem subject to ($\alpha , 0$)-FPAIF}]
\begin{equation}\label{box:fairfp}
\begin{aligned}
& \min_{\psi \, \in \, \Delta(\Hs)^\F, \, \gamma \, \in \, [0,1]}  & & \err \left( \psi ; \Px, \Q \right) \\
& \ \text{ s.t. $\forall x \in \X$:} & & \left\vert \Epsilonfp \left(x, \psi ; \Q \right) - \gamma \right\vert \le \alpha
\end{aligned}
\end{equation}
\end{tcolorbox}

\medskip
\begin{definition}\label{def:optfp}
Consider the optimization problem (\ref{box:fairfp}). Given distributions $\Px$ and $\Q$ and fairness approximation parameter $\alpha$, we denote the optimal solutions of (\ref{box:fairfp}) by $\psi^\star \left( \alpha ; \Px , \Q \right)$ and $\gamma^\star \left( \alpha ; \Px , \Q \right)$, and the value of the objective function at these optimal points by $\opt \left(\alpha; \Px, \Q \right)$. In other words
\begin{equation*}
\opt \left (\alpha; \Px, \Q \right) = \err \left( \psi^\star ; \Px, \Q \right)
\end{equation*}
\end{definition}

It is important to observe that the optimization problem (\ref{box:fairfp}) has a nonempty set of feasible points for any $\alpha$ and any distributions $\Px$ and $\Q$ because the following is always feasible: $\gamma = 0$ and $\psi_f = h^0$ for all $f \in \F$ where $h^0$ is the all-zero constant classifier. Since the distributions $\Px$ and $\Q$ are generally not known, we instead solve the \emph{empirical} version of (\ref{box:fairfp}). Consider a training data set consisting of $n$ individuals $X = \{x_i\}_{i=1}^n$ drawn independently from $\Px$ and $m$ problems $F=\{f_j\}_{j=1}^m$ drawn independently from $\Q$. We formulate the empirical fair learning problem in (\ref{box:fairermfp}) where we find an optimal fair mapping of the training problems $F$ to $\Delta (\Hs)$ given the individuals $X$. 

\begin{tcolorbox}[title= {Empirical Fair Learning Problem}]
\begin{equation}\label{box:fairermfp}
\begin{aligned}
& \ \ \min_{\pvec \, \in \, \Delta(\Hs)^m, \, \gamma \, \in \, [0,1]}  & & \err \left( \pvec ; \Pxhat, \Qhat\right) \\
& \text{ s.t. $\forall i \in \{1, \ldots, n\}$:} & &\left\vert \Epsilonfp \left( x_i , \pvec ; \Qhat\right) - \gamma \right\vert \le 2 \alpha \\
\end{aligned}
\end{equation}
\end{tcolorbox}

As also discussed in Section \ref{sec:learningaif}, solving the empirical problem (\ref{box:fairermfp}) will only give us a -- restricted -- mapping $\pvec = \psi \vert_F$ by which we mean we learn $\psi$ only on the finite domain $F \subseteq \F$. It is not clear for now how we can extend the restricted mapping $\psi \vert_F \in \Delta (\Hs)^m$ to a mapping $\psi \in \Delta (\Hs)^\F$ defined over the entire function space $\F$; however, we will see the specific form of our algorithm (as in the AIF setting: learning a set of weights over the training individuals) will allow us to come up with such an extension.

We use the dual perspective of constrained optimization problems followed by no regret dynamics to reduce (\ref{box:fairermfp}) to a two-player zero-sum game between a Learner (primal player) and an Auditor (dual player), and design an iterative algorithm to get an approximate equilibrium of the game. To do so, we first rewrite the constraints of (\ref{box:fairermfp}) in $\rvecfp \left( \pvec , \gamma; \Qhat \right) \le 0$ form where
\begin{equation}\label{eq:rfp}
\rvecfp \left( \pvec , \gamma ; \Qhat \right) = \begin{bmatrix} \Epsilonfp \left(x_i, \pvec; \Qhat \right) - \gamma - 2\alpha \\ \gamma - \Epsilonfp \left(x_i, \pvec; \Qhat\right) - 2 \alpha \end{bmatrix}_{i=1}^n \in \R^{2n}
\end{equation}
represents the ``fairness violations" of the pair $\left( \pvec, \gamma \right)$ in one single vector. Let the corresponding dual variables be represented by $\lamb = \left[\lamb_i^+, \lamb_i^- \right]_{i=1}^n \in \Lambda$, where $\Lambda = \{ \lamb \in  \R^{2n}_+ \, \vert \, || \lamb ||_1 \le B \}$. Using Equation (\ref{eq:rfp}) and the introduced dual variables, we have that the Lagrangian of (\ref{box:fairermfp}) is
\begin{align}\label{eq:lagrangianfp}
\Ls \left(\pvec, \gamma, \lamb \right) &=\err \left( \pvec ; \Pxhat, \Qhat\right) + \lamb^T \rvecfp \left( \pvec , \gamma ; \Qhat \right)
\end{align}
Therefore, we focus on solving the following minmax problem:
\begin{equation}\label{eq:minmaxfp}
\min_{\pvec \, \in \, \Delta(\Hs)^m, \, \gamma \, \in \, [0,1]} \, \max_{\lambda \in \Lambda} \ \Ls \left(\pvec, \gamma, \lamb \right)  \ = \ \max_{\lambda \in \Lambda} \, \min_{\pvec \, \in \, \Delta(\Hs)^m, \, \gamma \, \in \, [0,1]} \  \Ls \left(\pvec, \gamma, \lamb \right)
\end{equation}
Using no regret dynamics, an approximate equilibrium of this zero-sum game (i.e. a saddle point of $\Ls$) can be found in an iterative framework. In each iteration, we let the dual player run the \emph{exponentiated gradient descent} algorithm and the primal player run an approximate version of her \emph{best response} using the oracle $CSC(\Hs)$ that solves cost sensitive classification problems in $\Hs$ (see Definition \ref{def:csc}). In the following subsection, we will first describe the best response problem of the Learner, and show how the best response depends on the estimates $ \rhohatvec = [ \rhohat_i ]_{i=1}^n \in \R^n$ -- where $\rhohat_i$ is the fraction of problems in $F$ that maps $x_i$ to $0$ -- in addition to the weights $\wvec = \left[ \lambda_i^+ - \lambda_i^- \right] \in \R^n$ maintained by the Auditor. We will then argue that to avoid injecting correlation into the algorithm (so as to argue later on about generalization) we have to ``perturb" the best response of the Learner by using some other set of estimates $\rhotildevec = [ \rhotilde_i ]_{i=1}^n$ which is given to the algorithm and is independent of $F$. This is why the Learner is basically using an \emph{approximate} version of her best response.
\medskip

\begin{definition}\label{def:rho}
For an individual $x \in \X$, let $\rho_x$ and $\rhohat_x$ represent the probability that $x$ is labelled $0$ by a randomly sampled function $f \sim \Q$ and $f \sim \Qhat$, respectively. In other words
$$
\rho_{x} = \underset{f \sim \Q}{\Ps} \left[ f(x) = 0\right] \quad , \quad \rhohat_{x} = \underset{f \sim \Qhat}{\Ps} \left[ f(x) = 0\right]
$$
we will use $\rho_i \equiv \rho_{x_i}$ and $\rhohat_i \equiv \rhohat_{x_i}$ to denote the corresponding probabilities for $x_i$ in the training set of individuals $X$.
\end{definition}
\medskip

\begin{remark}\label{rem:rho}
Observe that using the introduced notation, for a mapping $\psi$ and $x\in \X$,
$$
\Epsilonfp \left(x, \psi; \Q \right) = \left( \frac{1}{\rho_x} \right)  \underset{f \sim \Q}{\E} \left[ \underset{h \sim \psi_f}{\Ps} \left[ h(x) = 1, \, f(x) = 0 \right] \right]
$$
$$
\Epsilonfn \left(x, \psi ; \Q \right) = \left( \frac{1}{1 - \rho_x} \right) \underset{f \sim \Q}{\E} \left[ \underset{h \sim \psi_f}{\Ps} \left[ h(x) =0, \, f(x) = 1 \right] \right]
$$
and that $\Epsilon \left(x,\psi; \Q \right)$ can be written as a linear combination of $\Epsilonfp \left(x,\psi; \Q \right)$ and $ \Epsilonfn \left(x,\psi; \Q \right)$:
$$
\Epsilon \left(x,\psi; \Q \right) = \rho_x \cdot \Epsilonfp \left(x,\psi; \Q \right) + \left(1-\rho_x \right) \cdot \Epsilonfn \left(x,\psi; \Q \right)
$$
\end{remark}

\subsection{$\bestfp$: The Learner's approximate Best Response}\label{subsec:bestresponsefp}
In each iteration of the algorithm, the Learner is given some $\lamb$ of the dual player and she wants to solve the following minimization problem.
\begin{align*}
&\argmin_{\pvec \, \in \, \Delta(\Hs)^m, \, \gamma \, \in \, [0,1]} \, \Ls \left( \pvec, \gamma, \lamb \right) \\
\equiv &\argmin_{\pvec \, \in \, \Delta(\Hs)^m, \, \gamma \, \in \, [0,1]} \, \err \left( \pvec ; \Pxhat, \Qhat\right) + \sum_{i=1}^n \left\{ \lambda_{i}^+ \left( \Epsilonfp \left(x_i, \pvec; \Qhat \right) - \gamma \right) + \lambda_{i}^- \left( \gamma  - \Epsilonfp \left(x_i, \pvec; \Qhat \right) \right) \right\} \\
\equiv &\argmin_{\pvec \, \in \, \Delta(\Hs)^m, \, \gamma \, \in \, [0,1]} \, \frac{1}{n} \sum_{i=1}^n \Epsilon \left(x_i, \pvec ; \Qhat\right) + \sum_{i=1}^n \left\{ \lambda_{i}^+ \left( \Epsilonfp \left(x_i, \pvec; \Qhat \right) - \gamma  \right) + \lambda_{i}^- \left( \gamma  - \Epsilonfp \left(x_i; \pvec; \Qhat \right)  \right) \right\}
\end{align*}
Let $\rhohat_i$ be defined as in Definition \ref{def:rho}. We have that by Remark \ref{rem:rho}
$$
\Epsilon \left(x_i, \pvec ; \Qhat\right) = \rhohat_{i} \cdot \Epsilonfp \left(x_i, \pvec; \Qhat \right) + \left( 1 - \rhohat_{i} \right) \cdot \Epsilonfn \left(x_i, \pvec; \Qhat \right)
$$
Now let $w_{i} = \lambda_{i}^+ - \lambda_{i}^-$ for all $i$ (Accordingly let $\wvec = \left[w_{1}, \ldots, w_{n} \right]^\top \in \R^n$). We have that the above minimization problem is equivalent to
\begin{align*}
\equiv &\argmin_{\pvec \, \in \, \Delta(\Hs)^m, \, \gamma \, \in \, [0,1]} \, - \gamma \sum_{i=1}^n w_{i} + \sum_{i=1}^n \left\{ \left(\frac{\rhohat_{i}}{n} + w_{i} \right) \Epsilonfp \left(x_i, \pvec; \Qhat \right) +  \left(\frac{1 - \rhohat_{i}}{n} \right) \Epsilonfn \left(x_i, \pvec; \Qhat \right) \right\}
\end{align*}
Now we can use the fact that
$$
\Epsilonfp \left(x_i, \pvec; \Qhat \right) = \frac{1}{m \rhohat_{i}} \sum_{j=1}^m \underset{h_j \sim \, p_j}{\Ps} \left[ h_j (x_i) = 1, f_j(x_i) = 0 \right]
$$
$$
\Epsilonfn \left(x_i, \pvec; \Qhat \right) = \frac{1}{m (1 - \rhohat_{i})} \sum_{j=1}^m \underset{h_j \sim \, p_j}{\Ps} \left[ h_j (x_i) = 0, f_j(x_i) = 1 \right]
$$
to conclude that the best response of the Learner is equivalent to minimizing the following function over the space of $(\pvec, \gamma) \in \Delta(\Hs)^m \times [0,1]$.
\begin{align*}
& -\gamma \sum_{i=1}^n w_{i} \\
& + \frac{1}{m}\sum_{j=1}^m  \left\{ \sum_{i=1}^n \left(\frac{1}{n} +  \frac{w_{i}}{\rhohat_{i}}\right) \underset{h_j \sim \, p_j}{\Ps} \left[ h_j (x_i) = 1, f_j(x_i) = 0 \right] + \left( \frac{1}{n} \right) \underset{h_j \sim \, p_j}{\Ps} \left[ h_j (x_i) = 0, f_j(x_i) = 1 \right] \right\}
\end{align*}

Therefore, as in the AIF setting, the minimization problem of the Learner gets nicely decoupled into $(m+1)$ disjoint minimization problems. First, the optimal value for $\gamma$ is chosen according to
\begin{equation}
\gamma = \1 \left[ \sum_{i=1}^n w_{i} > 0\right]
\end{equation}
and that for learning problem $j$, the following cost sensitive classification problem must be solved.
\begin{equation}
h_j = \argmin_{h \in \Hs} \, \sum_{i=1}^n \left\{ c_{i,j}^1 \, h (x_i) + c_{i,j}^0 \left( 1 - h (x_i) \right) \right\}
\end{equation}
where the costs are
$$c_{i,j}^1 = \left( \frac{1}{n} +  \frac{w_{i}}{ \rhohat_{i}} \right) \left(1 - f_j (x_i) \right) \quad , \quad c_{i,j}^0 = \left( \frac{1}{n} \right) f_j (x_i)$$

One major distinction between the Learner's best response in the FPAIF setting versus that of the AIF setting is that the empirical quantities $\{ \rhohat_i \}_{i=1}^n$ (which is estimated using the data set $F$) appear in the costs of the CSC problems for the FPAIF setting. As a major consequence, the generalization (with respect to $\Q$) arguments we had in the AIF setting won't work in this section because now the labels $\{ f_j (x_i) \}_{i,j}$ and the estimates $\{ \rhohat_i \}_{i=1}^n$ are correlated. We therefore assume in this section that each individual $x_i \in X$ comes with an estimate $\rhotilde_i$ of the rate $\rho_i$ that is independent of the data set $F$. More precisely stated, we assume our algorithm has access to estimates $\{ \rhotilde_i \}_{i=1}^n$ such that for all $i \in [n]$,
$
\left\vert \rhotilde_i - \rho_i \right\vert \le \sqrt{\log \left( n \right) / 2m_0}
$
where $m_0$ (will be specified exactly in our proposed algorithm) will be essentially the number of fresh problems that the Auditor is using in each round of the Algorithm. In fact, similar to what we did for the AIF setting, we will randomly partition $F$ into $T$ batch of size $m_0$: $F = \{F_t\}_{t=1}^T$ and will let the Auditor to use only $F_t$ at round $t \in [T]$ of the algorithm to update the vector of fairness violations $\rvecfp$. Notice assuming access to the estimates $\{ \rhotilde_i \}_{i=1}^n$ is not farfetched because we can assume there was one more batch of $m_0$ problems, say $F_0$, and that the quantities $\{ \rhotilde_i \}_{i=1}^n$ were estimated using the batch $F_0$ which is independent of $F = \{ F_t \}_{t=1}^T$. The upper bound we required for the difference $\left\vert \rhotilde_i - \rho_i \right\vert$ will just simply follow from a Chernoff-Hoeffding's bound.
\medskip

\begin{assumption}\label{ass:rho}
For $m_0$ specified later on, we assume in this section that our algorithm has access to quantities $\{ \rhotilde_i \}_{i=1}^n$, where we have that for all $i \in [n]$:
$
\left\vert \rhotilde_i - \rho_i \right\vert \le \sqrt{\frac{\log \left( n \right)}{2m_0}}
$.
\end{assumption}
\medskip

Under Assumption \ref{ass:rho}, we now modify the best response of the Learner and let it use the estimates $\{ \rhotilde_i \}_{i=1}^n$ instead of $\{ \rhohat_i \}_{i=1}^n$. This will consequently make the learner to accumulate regret over the course of the algorithm and that this is why we will call it the \emph{approximate} best response of the Learner. We have the approximate best response of the Learner (called $\bestfp$) written in Subroutine \ref{bestresponsefp}.

\begin{subroutine}
\KwIn{dual weights $\wvec \in \R^n$, estimates $\rhotildevec \in \R^n$, training examples $S = \big\{ x_i, \left( f_j(x_i) \right)_{j=1}^m \big\}_{i=1}^n$}
\medskip
$\gamma \leftarrow \1 \left[ \sum_{i=1}^n w_{i} > 0\right]$ \\
\For{$j= 1, \ldots, m$}{
$c_i^1 \leftarrow (w_{i}/\rhotilde_i + 1/n) (1 - f_j(x_i)) $ for $i \in [n]$\\
$c_i^0 \leftarrow (1/n) f_j(x_i) $ for $i \in [n]$\\
 $D \leftarrow \{ x_i , c_i^1, c_i^0 \}_{i=1}^n$\\
$h_{j} \leftarrow CSC \left(\Hs; D \right)$
}
$\hvec \leftarrow \left( h_{1}, \, h_{2}, \, \ldots, \, h_{m} \right)$\\
\medskip
\KwOut{$\left( \hvec, \gamma \right)$}
\caption{$\bestfp$ -- approximate best response of the Learner in the FPAIF setting}
\label{bestresponsefp}
\end{subroutine}

\subsection{\textbf{FPAIF-Learn}: Implementation and In-sample Guarantees}\label{subsec:algorithmfp}
We implement the introduced game theoretic framework in Algorithm \ref{algofp} and call it \textbf{FPAIF-Learn}. The overall style of the algorithm is similar to \textbf{AIF-Learn} (Algorithm \ref{algo}) except that \textbf{FPAIF-Learn} takes a set of estimates $\rhotildevec \in \R^n$ as input and that $\rhotildevec$ is used in the approximate best response of the Learner. Once again, in the service of arguing for generalization, we split $F$ into $T$ batches of size $m_0$ uniformly at random and let the Auditor use only a fresh batch of $m_0$ problems in each round of the algorithm to update the dual variables $\lamb$, and accordingly the weights $\wvec$ over the individuals. This is reflected in the algorithm by writing $\Qhat_t = \uniform (F_t)$ for the uniform distribution over $F_t$ and $\hvec_t \vert_{F_t}$ for the corresponding learned classifiers of $F_t$. The algorithm will terminate after $T = O \left( \log \left(n\right)/ ( \nu^2 \alpha^2 ) \right)$ iterations and output the average plays of the Learner and the Auditor, along with a mapping $\psihat = \psihat \, (X, \rhotildevec, \What ) \in \Delta (\Hs)^\F$ which is the object we wanted to learn. In fact, $\psihat$ extends the learned restricted mapping $\pvechat = \psihat \vert_F$ of the algorithm from the finite domain $F$ to the entire space $\F$. We have the pseudocode for the mapping $\psihat$ written in detail in Mapping \ref{algo:psihatfp}.

\begin{algorithm}
\KwIn{fairness parameter $\alpha$, approximation parameter $\nu$, \\ \ \ \ \ \ \ \ \ \ \ \ estimates $\rhotildevec = [\rhotilde_i]_{i=1}^n$, \\ \ \ \ \ \ \ \ \ \ \ \ training data set $X = \{x_i\}_{i=1}^n$ and $F=\{f_j\}_{j=1}^m$ }
\medskip
\begin{center}
$B \leftarrow \frac{1+ 2 \nu}{ \alpha} , \ \ T \leftarrow  \frac{16 B^2 \left( 1 + 2\alpha  \right)^2  \log \left( 2n+1 \right) }{\nu^2}, \  \ \eta \leftarrow \frac{\nu}{4 \left( 1 + 2 \alpha \right)^2 B}, \ \ m_0 \leftarrow \frac{m}{T}, \ \ S \leftarrow \left\{ x_i, \left( f_j(x_i) \right)_{j=1}^m \right\}_{i=1}^n$
\end{center}
\medskip
Partition $F$ uniformly at random: $F = \{F_t\}_{t=1}^T$ where $| F_t | = m_0$.\\
$\boldsymbol{\theta}_1 \leftarrow \boldsymbol{0} \in \R^{2n}$\\
\For{$t=1, \ldots, T$}{
$\lambda_{i,t}^{\bullet} \leftarrow B \frac{\exp(\theta_{i,t}^{\bullet})}{1 + \sum_{i',\bullet'} \exp(\theta_{i',t}^{\bullet'})}$ for $i \in [n]$ and $\bullet \in \{ +, -\}$ \\
$\wvec_t \leftarrow [ \lambda_{i,t}^+ - \lambda_{i,t}^- ]_{i=1}^n \in \R^n$\\
$( \hvec_t, \gamma_t ) \leftarrow \bestfp (\wvec_t; \rhotildevec, S)$\\
$\boldsymbol{\theta}_{t+1} \leftarrow \boldsymbol{\theta}_{t} + \eta \cdot \rvecfp \left(\hvec_t \vert_{F_t} , \gamma_t ; \Qhat_t\right)$
}
$ \pvechat \leftarrow \frac{1}{T} \sum_{t=1}^T \hvec_t \ , \quad \gammahat \leftarrow \frac{1}{T} \sum_{t=1}^T \gamma_t \ , \quad \lambhat \leftarrow \frac{1}{T} \sum_{t=1}^T \lamb_t \ , \quad \What \leftarrow \{\wvec_t\}_{t=1}^T$
\medskip

\KwOut{average plays $\left( \pvechat, \, \gammahat, \, \lambhat \right)$, mapping $\psihat = \psihat \left(X, \rhotildevec, \What \right)$ (see Mapping \ref{algo:psihatfp})}
\caption{\textbf{FPAIF-Learn} -- learning subject to FPAIF}
\label{algofp}
\end{algorithm}

\begin{mapping}
\KwIn{$f \in \F$ (represented to $\psihat$ as $\{ f(x_i) \}_{i=1}^n$)}

\medskip
\For{$t=1, \ldots, T$}{
$c_i^1 \leftarrow (w_{i,t}/\rhotilde_i + 1/n) (1 - f(x_i)) $ for $i \in [n]$\\
$c_i^0 \leftarrow (1/n) f(x_i) $ for $i \in [n]$\\
 $D \leftarrow \{ x_i , c_i^1, c_i^0 \}_{i=1}^n$\\
$h_{f,\wvec_t} \leftarrow CSC \left(\Hs; D \right)$\\
}
$\psihat_{f} \leftarrow \frac{1}{T} \sum_{t=1}^T  h_{f,\wvec_t}$
\medskip

\KwOut{$\psihat_f \in \Delta (\Hs)$}
\caption{{$\boldsymbol{\psihat} \, ( X, \rhotildevec, \What )$} -- pseudocode for the mapping $\psihat$ output by Algorithm \ref{algofp}}
\label{algo:psihatfp}
\end{mapping}

The analysis of Algorithm \ref{algofp} follows the same style and uses the same ideas as in the AIF learning section. We will start by establishing the regret bounds for the Learner and the Auditor in Lemma \ref{lemma:regretlearnerfp} and Lemma \ref{lemma:regretdualfp}, respectively, and will show just as before how these regret bounds can be eventually turned into in-sample accuracy and fairness guarantees. One major distinction is that when working with false positive rates, the quantities $\rho_x, \rhohat_x, \rhotilde_x$ introduced before will show up in the algorithm's analysis. Define for the training individuals $X = \{x_i\}_{i=1}^n$
$$
\rhomin = \min_{i \in [n]} \rho_i \quad , \quad \rhohatmin = \min_{i \in [n]} \rhohat_i \quad , \quad \rhotildemin = \min_{i \in [n]} \rhotilde_i
$$
We will in fact see $\rhotildemin$ and $\rhohatmin$ will appear in the regret bounds of the Learner and that of the Auditor, respectively.
\medskip

\begin{lemma}[Regret of the Learner]\label{lemma:regretlearnerfp}
Let $0< \delta <1$. Let $\{  \hvec_t, \gamma_t \}_{t=1}^T$ be the sequence of approximate best response plays by the Learner to the given $\{ \lamb_t \}_{t=1}^T$ of the Auditor over $T$ rounds of Algorithm \ref{algofp}. We have that for any set of individuals $X$, with probability at least $1-\delta/2$ over the problems $F$, the (average) regret of the Learner is bounded as follows:
\begin{equation*}\label{eq:regretlearnerfp}
\frac{1}{T} \sum_{t=1}^T \Ls \left(\hvec_t, \gamma_t, \lamb_t \right) - \frac{1}{T} \min_{\pvec \, \in \Delta (\Hs)^m, \, \gamma \, \in [0,1]} \sum_{t=1}^T \Ls \left(\pvec, \gamma, \lamb_t \right) \, \le \, \frac{4 B}{\rhotildemin}\sqrt{\frac{\log\left( 8nT/ \delta \right)}{2 m_0}}
\end{equation*}
\end{lemma}

\medskip
\begin{lemma}[Regret of the Auditor]\label{lemma:regretdualfp}
Let $0< \delta <1$. Let $\{ \lamb_t \}_{t=1}^T$ be the sequence of exponentiated gradient descent plays (with learning rate $\eta$) by the Auditor to the given $\{ \hvec_t , \gamma_t \}_{t=1}^T$ of the Learner over $T$ rounds of Algorithm \ref{algofp}. We have that for any set of individuals $X$, with probability at least $1-\delta/2$ over the problems $F$, the (average) regret of the Auditor is bounded as follows: For any $\lamb \in \Lambda$,
\begin{align*}\label{eq:regretlambda}
&\frac{1}{T} \sum_{t=1}^T \Ls \left(\hvec_t, \gamma_t, \lamb \right) - \frac{1}{T} \sum_{t=1}^T \Ls \left(\hvec_t, \gamma_t, \lamb_t \right)  \le \frac{2B}{ \rhohatmin} \sqrt{ \frac{\log \left( 8nT / \delta \right)}{2 m_0} } + \frac{B \log \left(2n +1 \right)}{\eta T} + \eta B \left( 1 + 2 \alpha \right)^2
\end{align*}
\end{lemma}
\medskip

Observe that in order to control the regret of the Learner and the Auditor at level $ O (\nu)$ we need to assume that $m_0$ is large enough such that the regret bound of the Learner and the first term appearing in the regret bound of the Auditor are sufficiently small.
\medskip

\begin{assumption}\label{ass:mfp}
For a given confidence parameter $\delta$, inputs $\alpha$ and $\nu$ of Algorithm \ref{algofp}, we assume throughout this section that the number of fresh problems $m_0$ used in each round of Algorithm \ref{algofp} satisfies $m_0 \ge O \left(\frac{ \log \left( n T / \delta \right)}{\alpha^2 \nu^2 \rhomin^2 } \right)$, or equivalently $m = m_0 \cdot T  \ge O \left(\frac{ T \log \left( n T / \delta \right)}{\alpha^2 \nu^2 \rhomin^2 } \right)$.
\end{assumption}
\medskip

Note that Assumption \ref{ass:mfp} immediately implies via a Chernoff bound that $\rhohatmin \ge \rhomin/2$ and that it also implies via Assumption \ref{ass:rho} that $\rhotildemin \ge \rhomin/2$. In the following lemma we characterize the average play of the Learner and the Auditor. The proof of this lemma follows from the regret bounds developed in Lemma \ref{lemma:regretlearnerfp} (Regret of the Learner) and Lemma \ref{lemma:regretdualfp} (Regret of the Auditor) and uses exactly the same techniques as in the proof of Lemma \ref{thm:approxequilib}.
\medskip

\begin{lemma}[Average Play Characterization]\label{thm:approxequilibfp}
Let $0 < \delta < 1$. Let $\left(\pvechat, \gammahat, \lambhat \right)$ be the average plays output by Algorithm~\ref{algofp}. We have that under Assumption \ref{ass:mfp}, for any set of observed individuals $X$, with probability at least $1-\delta$ over the observed labelings $F$, the average plays $\left(\pvechat, \gammahat, \lambhat \right)$ forms a $\nu$-approximate equilibrium of the game, i.e.,
\begin{align*}
&\Ls \left(\pvechat, \gammahat, \lambhat \right) \, \le \, \Ls \left(\pvec, \gamma, \lambhat \right)  + \nu \quad \text{for all } \pvec \in \Delta (\Hs)^m \, , \gamma \in [0,1] \\
&\Ls \left(\pvechat, \gammahat, \lambhat \right) \, \ge \, \Ls \left(\pvechat, \gammahat, \lamb \right) - \nu \quad \text{for all } \lamb \in  \Lambda
\end{align*}
\end{lemma}

We conclude this subsection with our main Theorem \ref{thm:insampleguaranteesfp} that provides in-sample accuracy and fairness guarantees for the learned set of classifiers $\pvechat \in \Delta (\Hs)^m$ of Algorithm \ref{algofp}. This theorem follows immediately from Lemma \ref{thm:approxequilibfp} and the proof is pretty much similar in style to the proof of Theorem \ref{thm:insampleguarantees}.
\medskip
\begin{theorem}[In-sample Accuracy and Fairness]\label{thm:insampleguaranteesfp}
Let $0< \delta <1$ and suppose Assumption \ref{ass:mfp} holds. Let $\left(\pvechat, \gammahat \right)$ be the output of Algorithm~\ref{algofp} and let $\left( \pvec, \gamma \right)$ be any feasible pair of variables for the empirical fair learning problem (\ref{box:fairermfp}). We have that for any set of individuals $X$, with probability at least $1-\delta$ over the labelings $F$,
\begin{align*}
\err  \left(\pvechat; \Pxhat, \Qhat \right) \, \le \, \err \left(\pvec; \Pxhat, \Qhat \right) + 2 \nu
\end{align*}
and that $\pvechat$ satisfies $(3\alpha, 0)$-FPAIF with respect to $(\Pxhat, \Qhat)$. In other words, for all $i \in [n]$,
\begin{align*}
\left\vert \Epsilonfp \left(x_i, \pvechat; \Qhat \right) - \gammahat \right\vert \, \le \, 3\alpha
\end{align*}
\end{theorem}

\subsection{Generalization Theorems}\label{subsec:generalizationfp}
We now consider the mapping $\psihat$ learned by Algorithm \ref{algofp} and study the generalization bounds both for accuracy and fairness. As in the AIF learning setting, we state our generalization theorems in three steps. We first consider in Theorem \ref{thm:generalizationxfp} the empirical distribution of the problems $\Qhat$ and see how we can lift the guarantees from $\Pxhat$ into the true underlying distribution of individuals $\Px$. We will then consider generalization only over the problem generating distribution $\Q$ in Theorem \ref{thm:generalizationffp}. We will eventually in Theorem \ref{thm:generalizationxffp} provide accuracy and fairness guarantees for the learned mapping $\psihat$ with respect to the distributions $( \Px, \Q)$. We will use $\opt$ (defined formally in Definition \ref{def:optfp}) as a benchmark to evaluate the accuracy of the mapping $\psihat$. Proofs of these theorems are similar to their counterparts in the AIF section.
 \medskip
 
\begin{theorem}[Generalization over $\Px$]\label{thm:generalizationxfp}
Let $0 < \delta < 1$. Let $\left( \psihat, \gammahat \right)$ be the outputs of Algorithm \ref{algofp}, and suppose
$$
n \ge \widetilde{O} \left( \frac{m \, d_\Hs + \log \left(1/\nu^2 \delta\right)}{\alpha^2 \beta^2} \right)
$$
where $d_\Hs$ is the VC dimension of $\Hs$. We have that with probability at least $1-5\delta$ over the observed data set $(X,F)$, the mapping $\psihat$ satisfies $\left( 5\alpha, \beta \right)$-FPAIF with respect to $\left( \Px, \Qhat \right)$, i.e.,
$$
\underset{x \, \sim \Px}{\Ps} \left( \left\vert \Epsilonfp \left( x, \psihat; \Qhat \right) - \gammahat \right\vert > 5 \alpha  \right) \le \beta
$$
and that,
$$
\err \left( \psihat ; \Px , \Qhat \right) \le \opt \left( \alpha ; \Px, \Qhat \right) +  O \left( \nu \right) + O \left( \alpha \beta \right)
$$
\end{theorem}
\medskip

\begin{theorem}[Generalization over $\Q$]\label{thm:generalizationffp}
Let $0 < \delta < 1$. Let $\left( \psihat, \gammahat \right)$ be the outputs of Algorithm \ref{algofp} and suppose
$$
m \ge \Otilde \left( \frac{\log \left( n \right) \log \left( n/\delta \right)}{\rhomin^2 \nu^4 \alpha^4} \right)
$$
where $\rhomin = \min_{i \in [n]} \rho_i$. We have that for any set of observed individuals $X$, with probability at least $1-6\delta$ over the observed problems $F$, $\psihat$ satisfies $\left( 4\alpha, 0 \right) $-FPAIF with respect to $\left( \Pxhat, \Q \right)$, i.e.,
$$
\underset{x \, \sim \Pxhat}{\Ps} \left( \left\vert \Epsilonfp \left( x, \psihat; \Q \right) - \gammahat \right\vert > 4\alpha  \right) = 0
$$
and that,
$$
\err \left( \psihat ; \Pxhat , \Q \right) \le \opt \left( \alpha ; \Pxhat, \Q \right) + O \left( \nu \right)
$$
\end{theorem}
\medskip

\begin{theorem}[Simultaneous Generalization over $\Px$ and $\Q$]\label{thm:generalizationxffp}
Let $0 < \delta < 1$. Let $\left( \psihat, \gammahat \right)$ be the outputs of Algorithm \ref{algofp} and suppose
$$
n \ge \widetilde{O} \left( \frac{m \, d_\Hs + \log \left(1/\nu^2 \delta\right)}{\alpha^2 \beta^2} \right) \quad , \quad m \ge \Otilde \left( \frac{\log \left( n \right) \log \left( n/\delta \right)}{\rho_{\text{inf}}^2 \, \nu^4 \alpha^4} \right)
$$
where $d_\Hs$ is the VC dimension of $\Hs$ and $\rho_{\text{inf}} = \inf_{x \in \X} \rho_x$. We have that with probability at least $1-12 \delta$ over the observed data set $(X,F)$, the learned mapping $\psihat $ satisfies $\left( 6\alpha, 2 \beta \right)$-FPAIF with respect to the distributions $\left( \Px, \Q \right)$, i.e.,
$$
\underset{x \, \sim \Px}{\Ps} \left( \left\vert \Epsilonfp \left( x, \psihat ; \Q \right) - \gammahat \right\vert > 6 \alpha  \right) \le 2\beta
$$
and that,
$$
\err \left( \psihat ; \Px , \Q \right) \le \opt \left( \alpha ; \Px, \Q \right) + O \left( \nu \right) + O \left( \alpha \beta \right)
$$
\end{theorem}